%% file: main.tex
\documentclass{article}

\usepackage{microtype}
\usepackage{graphicx}
\usepackage{subcaption}
\usepackage{booktabs} 
\usepackage{hyperref}

\usepackage[preprint]{icml2026}

\usepackage{amsmath}
\usepackage{amssymb}
\usepackage{mathtools}
\usepackage{amsthm}
\usepackage{multirow}

\usepackage[capitalize,noabbrev]{cleveref}

\theoremstyle{plain}
\newtheorem{theorem}{Theorem}[section]

\newtheorem{lemma}[theorem]{Lemma}

\theoremstyle{definition}

\theoremstyle{remark}

\usepackage[textsize=tiny]{todonotes}
\newcommand{\draft}[1]{}

\icmltitlerunning{Multi-Head Attention as a Source of Catastrophic Forgetting in MoE Transformers}

\begin{document}

\twocolumn[
  \icmltitle{Multi-Head Attention as a Source of Catastrophic Forgetting \\in MoE Transformers}



  \icmlsetsymbol{equal}{*}

  \begin{icmlauthorlist}
    \icmlauthor{Anrui Chen}{fdu}
    \icmlauthor{Ruijun Huang}{fdu}
    \icmlauthor{Xin Zhang}{fdu}
    \icmlauthor{Fang Dong}{fdu}
    \icmlauthor{Hengjie Cao}{fdu}
    \icmlauthor{Zhendong Huang}{fdu}
    \icmlauthor{Yifeng Yang}{fdu}
    \icmlauthor{Mengyi Chen}{fdu}
    \icmlauthor{Jixian Zhou}{fdu}
    \icmlauthor{Mingzhi Dong}{bath}
    \icmlauthor{Yujiang Wang}{oxford}
    \icmlauthor{Jinlong Hou}{sii}
    \icmlauthor{Qin Lv}{boulder}
    \icmlauthor{Robert P. Dick}{michigan}
    \icmlauthor{Yuan Cheng}{sii}
    \icmlauthor{Tun Lu}{fdu}
    \icmlauthor{Fan Yang}{fdu}
    \icmlauthor{Li Shang}{fdu}
  \end{icmlauthorlist}

  \icmlaffiliation{fdu}{Fudan University, Shanghai, China}
  \icmlaffiliation{bath}{University of Bath, Bath, United Kingdom}
  \icmlaffiliation{oxford}{Oxford Suzhou Centre for Advanced Research, Suzhou, China}
  \icmlaffiliation{sii}{Shanghai Innovation Institute, Shanghai, China}
  \icmlaffiliation{boulder}{Department of Computer Science, University of Colorado Boulder, Colorado, USA}
  \icmlaffiliation{michigan}{Department of Electrical Engineering and Computer Science, University of Michigan}

  \icmlcorrespondingauthor{Li Shang}{lishang@fudan.edu.cn}

  \icmlkeywords{Machine Learning, ICML}

  \vskip 0.3in
]



\printAffiliationsAndNotice{}  

\begin{abstract}
Mixture-of-Experts (MoE) architectures are often considered a natural fit for continual learning because sparse routing should localize updates and reduce interference, yet MoE Transformers still forget substantially even with sparse, well-balanced expert utilization. We attribute this gap to a pre-routing bottleneck: multi-head attention concatenates head-specific signals into a single post-attention router input, forcing routing to act on co-occurring feature compositions rather than separable head channels. We show that this router input simultaneously encodes multiple separately decodable semantic and structural factors with uneven head support, and that different feature compositions induce weakly aligned parameter-gradient directions; as a result, routing maps many distinct compositions to the same route. We quantify this collision effect via a route-wise effective composition number \(N_{\mathrm{eff}}\) and find that higher \(N_{\mathrm{eff}}\) is associated with larger old-task loss increases after continual training. Motivated by these findings, we propose \textsc{MH-MoE}, which performs head-wise routing over sub-representations to increase routing granularity and reduce composition collisions. On TRACE with Qwen3-0.6B/8B, \textsc{MH-MoE} effectively mitigates forgetting, reducing \(-\mathrm{BWT}\) on Qwen3-0.6B from \(11.2\%\) (\textsc{LoRAMoE}) to \(4.5\%\).
\end{abstract}

\input{section/intro}

\input{section/analysis}
\input{section/method}
\input{section/related_work}

\input{section/experiment}


\input{main.bbl}

\newpage
\appendix
\onecolumn
\input{section/appendix}



\end{document}

%% file: section/intro.tex
\section{Introduction}
\label{sec:intro}

Mixture-of-Experts (MoE) architectures \cite{jacobs1991adaptive,jordan1994hierarchical} are appealing for continual and multi-task learning because routing can localize updates to a subset of experts, potentially reducing gradient interference and thereby mitigating catastrophic forgetting \cite{chen2023lifelong,kang2025self,li2024moe,li2024theory,dou2024loramoe}. In principle, expert specialization should preserve previously learned knowledge by isolating task-specific parameter changes \cite{ramasesh2020anatomy,davari2022probing,goodfellow2013empirical}.

Despite this promise, our empirical results show that MoE Transformers continue to suffer substantial catastrophic forgetting, even when expert utilization is sparse and well-balanced. On TRACE, the MoE baseline attains markedly negative backward transfer (BWT $=-11.2\%$ on Qwen3-0.6B; Table~\ref{tab:trace_final_per_task}), reflecting systematic degradation on earlier tasks as later tasks are learned. Expert modularity alone is therefore insufficient to prevent interference. A fundamental question follows: if experts are intended to isolate knowledge, where does task interference actually arise in MoE Transformers?

In this work, we argue that catastrophic forgetting in MoE Transformers is primarily caused by a structural bottleneck introduced before expert routing: multi-head attention. In standard Transformer architectures, outputs from multiple representation heads—each encoding distinct and heterogeneous features—are concatenated into a single vector prior to routing. This design implicitly assumes that the concatenated representation forms a coherent feature space suitable for expert selection.

Our analysis shows that this assumption is generally violated: multi-head attention produces a post-attention router input in which multiple feature signals co-occur in a single vector, while their support is unevenly distributed across heads. As a result, MoE routing must make expert decisions based on feature co-occurrence, which is prone to composition collisions and interference.
Concretely, we establish three findings:

\textbf{Multi-head attention mixes head-structured feature signals.}
The post-attention router input jointly encodes multiple separately decodable semantic and structural features (Fig.~\ref{fig:probe_and_overlap}), while their support is highly non-uniform across representation heads (Fig.~\ref{fig:head_feature_align}), showing that head-specific feature channels are aggregated into a single vector before routing.

\textbf{Feature compositions induce diverse learning signals.}
Different feature compositions produce composition-conditioned parameter-gradient directions with low cosine agreement (Fig.~\ref{fig:grad_alignment}), implying that a single shared update direction cannot align well with many compositions simultaneously.

\textbf{Composition collisions under MoE routing amplify forgetting.}
Because MoE routing compresses the multiplexed router input into a single expert decision, many distinct feature compositions collide on the same route (high route-wise effective composition number \(N_{\mathrm{eff}}\); Fig.~\ref{fig:layer_neff}). Routes with higher \(N_{\mathrm{eff}}\) exhibit larger old-task loss increases after continual training, linking composition mixing to catastrophic forgetting (Fig.~\ref{fig:neff_binplot}).

Motivated by these findings, we propose \textsc{MH-MoE}, which performs routing \emph{independently across multiple representation heads} rather than making a single expert decision from a head-mixed router input. This head-wise routing yields two practical benefits:

\textit{Mitigates forgetting with minimal accuracy loss.} Head-wise routing reduces feature-composition collisions within each update destination (lower route-wise \(N_{\mathrm{eff}}\)).
Because different compositions induce weakly aligned learning signals, this reduces gradient conflict and improves retention.

\textit{Task-agnostic and streamable.} \textsc{MH-MoE} does not rely on task boundaries, task IDs, or replay buffers, and uses the same token-level routing mechanism during training and inference, making it naturally compatible with continuous streams where task switches are unknown.

We evaluate \textsc{MH-MoE} on TRACE (8 tasks) with pretrained Qwen3-0.6B/8B backbones, comparing against a standard MoE baseline (\textsc{LoRAMoE}). 
\textsc{MH-MoE} consistently improves the retention--accuracy tradeoff: on Qwen3-0.6B, OP increases from 0.378 to 0.467 and BWT improves from $-0.112$ to $-0.045$; on Qwen3-8B, OP increases from 0.551 to 0.569 and BWT improves from $-0.055$ to $-0.051$.

%% file: section/analysis.tex
\section{Analysis}
\label{sec:analysis}

In this section, we explain why MoE Transformers can still suffer substantial catastrophic forgetting. All analyses use Qwen3-0.6B (and its MoE variant) on C-STANCE and FOMC datasets.

\subsection{Post-Attention Router Inputs Are Head-Mixed and Multi-feature}
\label{sec:attn_output_multifeature}

This subsection asks whether the router input provides separable signals that routing can exploit, or instead mixes multiple factors so that routing must act on their co-occurrence. We answer this with two analyses:
(i) we test which semantic/structural variables are linearly decodable from the post-attention representation and whether their probe-induced subspaces overlap;
(ii) we quantify whether these signals are supported unevenly across representation heads.

\textbf{Features correspond to linearly decodable signals in the representation.}
We define a \emph{feature} $Y$ as a discrete variable whose value is linearly decodable from the post-attention router input $h_t^{(\ell)}\in\mathbb{R}^d$. For each feature $Y$ and layer $\ell$, we train a multinomial linear probe on frozen representations:
\begin{equation}
\hat{p}_\ell(y\mid h)
=\mathrm{Softmax}\!\left(W_Y^{(\ell)} h + b_Y^{(\ell)}\right).
\label{eq:probe_attn_out}
\end{equation}
The probe induces a feature-specific decoding geometry via the \emph{decoding subspace}
\begin{equation}
\mathcal{S}_{Y}^{(\ell)} \;=\; \mathrm{span}\!\left( (W_Y^{(\ell)})^\top \right)\subseteq\mathbb{R}^d,
\label{eq:feature_subspace}
\end{equation}
which captures the linear directions in \(h_t^{(\ell)}\) predictive of \(Y\).
We instantiate \(Y\) using salient semantic and structural variables: domain identity, stance label, token-frequency bucket, and relative-position bucket.

\textbf{Multiple features are mixed in a single router input.}

All studied variables are predicted substantially above chance across layers (Fig.~\ref{fig:probe-across-layers}),
showing that the same router input \(h_t^{(\ell)}\) simultaneously carries semantic (domain/stance) and structural (frequency/position) information.
Moreover, probe-induced decoding subspaces have small pairwise overlap (Fig.~\ref{fig:f_overlap}),
suggesting these signals rely on largely distinct linear directions rather than a shared low-dimensional factor.
Together, \(h_t^{(\ell)}\) is \emph{multiplexed}: multiple separable signals co-exist in one vector, so a single routing score computed from \(h_t^{(\ell)}\) must implicitly trade off among them when they co-occur.

\begin{figure}[ht]
  \centering

  \begin{subfigure}[t]{0.48\linewidth}
    \centering
    \includegraphics[width=\linewidth]{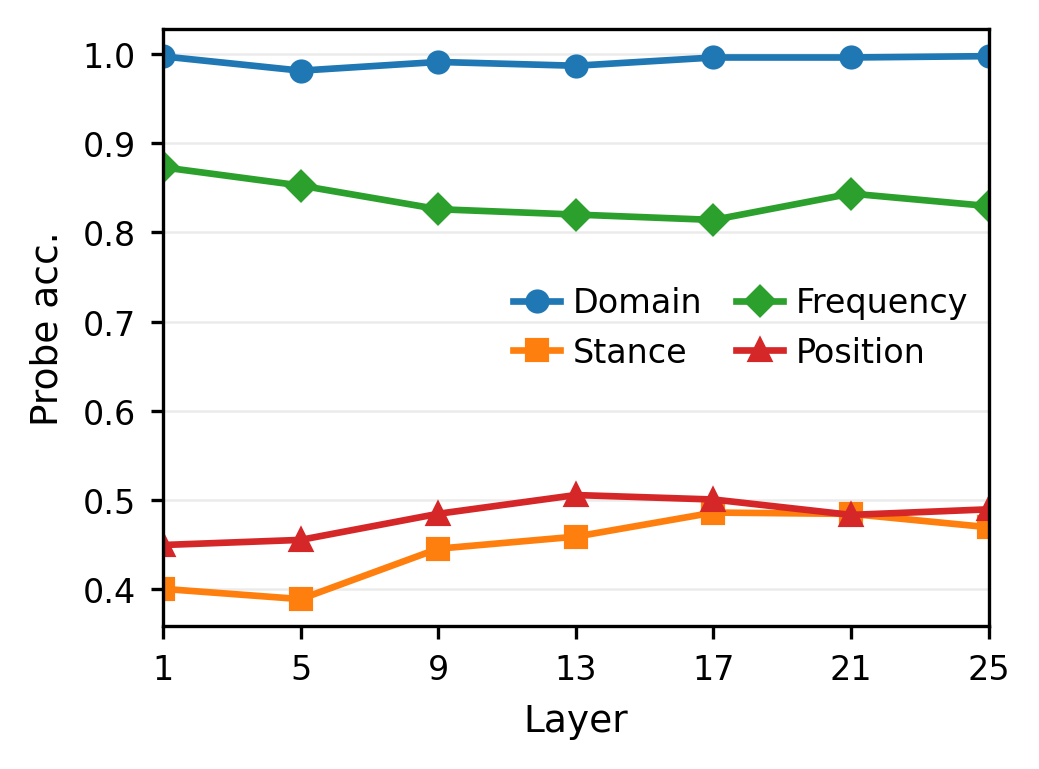}
    \caption{}
    \label{fig:probe-across-layers}
  \end{subfigure}
  \hspace{0.01\linewidth}
  \begin{subfigure}[t]{0.48\linewidth}
    \centering
    \includegraphics[width=\linewidth]{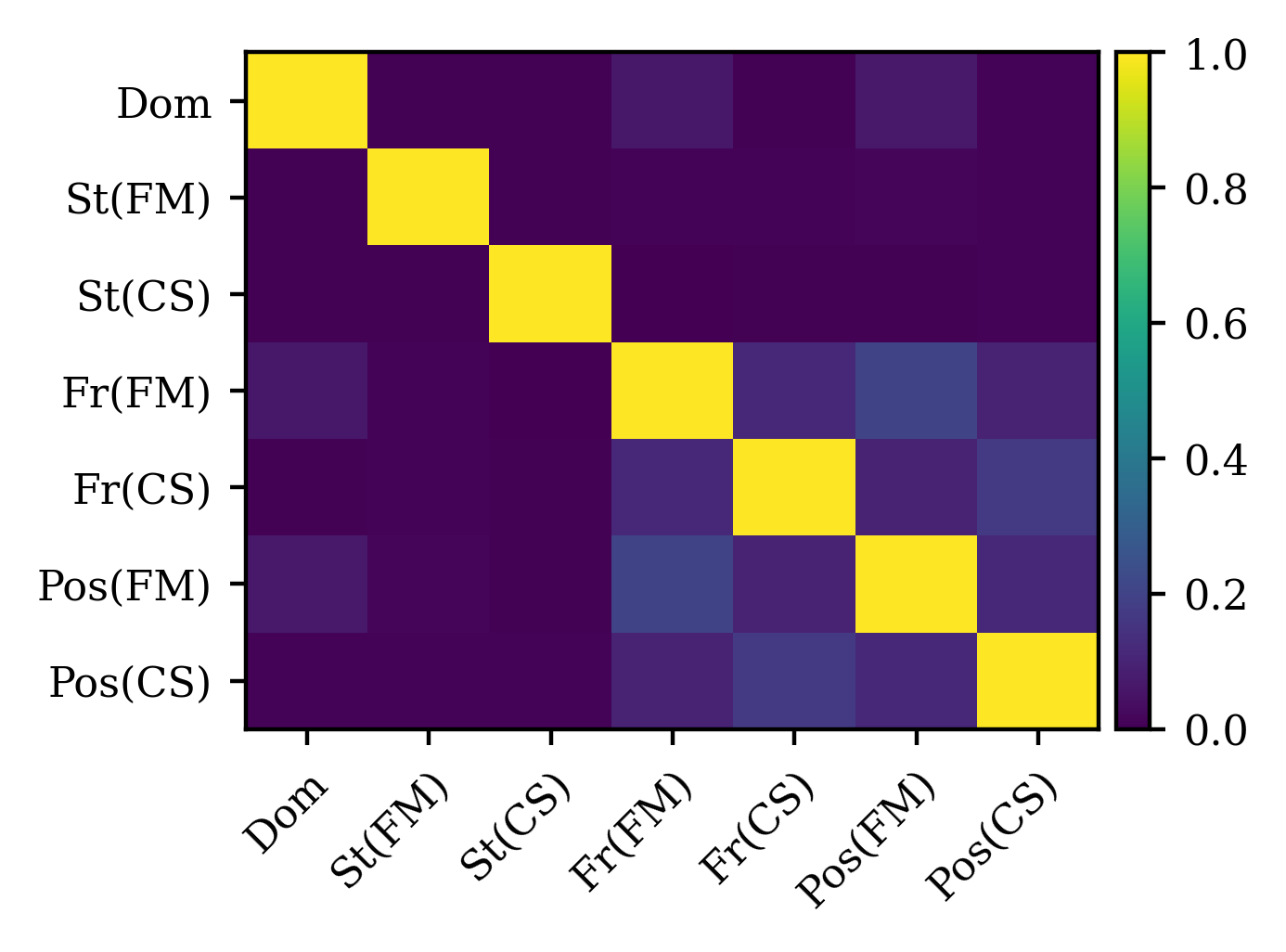}
    \caption{}
    \label{fig:f_overlap}
  \end{subfigure}

  \caption{\textbf{The router input multiplexes multiple decodable features.}
(a) Linear probes trained on post-attention states \(h_t^{(\ell)}\) predict domain/stance (semantic) and frequency/position (structural) well above chance across layers, showing that these signals co-exist in the same vector.
(b) Overlap between probe-induced decoding subspaces is small, indicating that the co-existing signals occupy largely distinct linear directions within \(h_t^{(\ell)}\).}

  \label{fig:probe_and_overlap}
\end{figure}

\textbf{Features are head-structured.}
The multiplexing in \(h_t^{(\ell)}\) is not uniform across representation heads.
For each feature \(Y\), we quantify head-wise \emph{ablation-based importance}: we remove one head at the ablation site and measure how much the feature’s probe performance degrades on the router input (Appendix~\ref{app:head_feature_metric}).
Normalizing these importance scores across heads yields a per-feature distribution over heads (\emph{shares}) that summarizes where the decodable signal is concentrated.
We observe highly non-uniform head--feature patterns (Fig.~\ref{fig:head_feature_align}): for each feature, a small subset of heads accounts for a disproportionate share of importance, and the dominant head subsets differ across features.

\begin{figure}[h]
  \centering
  \includegraphics[width=0.9\linewidth]{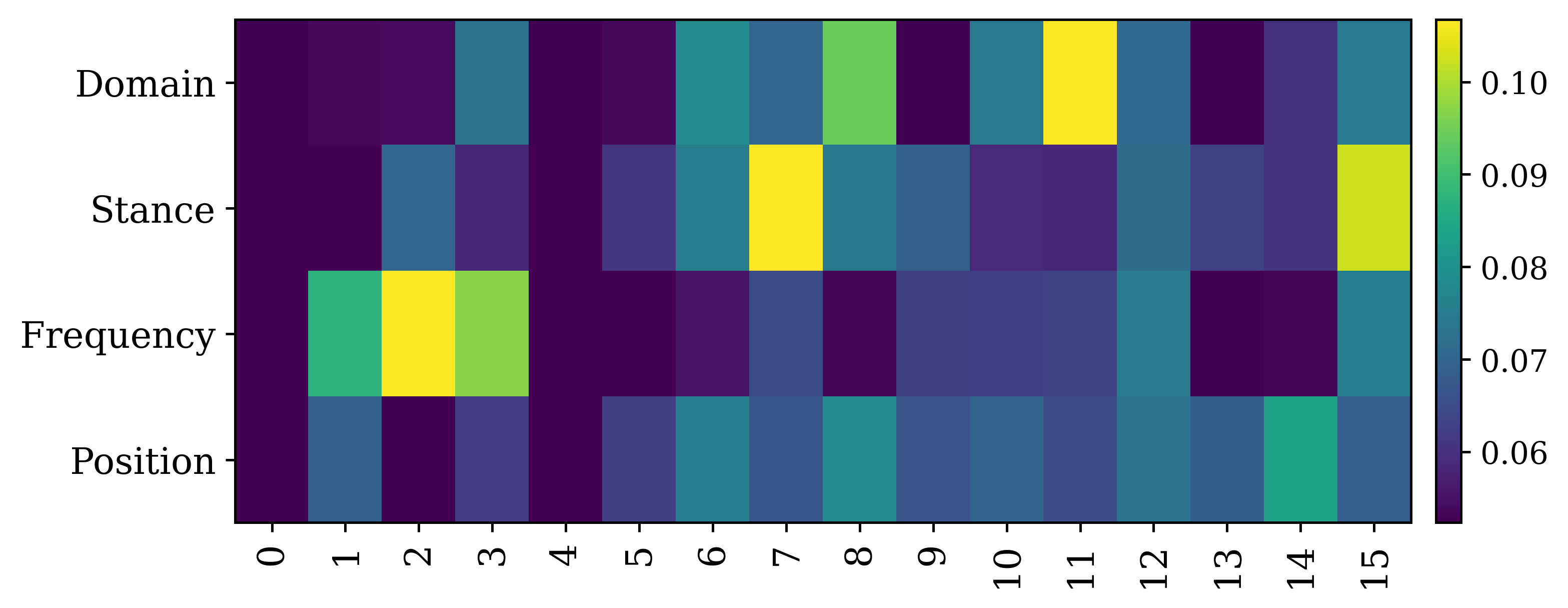}
  \caption{\textbf{Feature signals are head-structured but appear mixed in the router input.}
For each feature \(Y\), we estimate head-wise causal importance by ablating one head and measuring the resulting drop in probe accuracy on \(h_t^{(\ell)}\).
Importance is highly non-uniform across heads and differs by feature, suggesting that feature signals originate in specific heads but are multiplexed in the post-attention router input.}

  \label{fig:head_feature_align}
\end{figure}

These results show that the standard router input $h_t^{(\ell)}$ is both \emph{multiplexed} and \emph{head-structured}: multiple separable feature signals are present, but their support is uneven across heads and then aggregated into a single vector. Standard MoE routing compresses $h_t^{(\ell)}$ into a \emph{single} expert decision, so it cannot preserve head-wise separation when multiple signals co-occur. Consequently, tokens with different feature compositions can map to the same route, forcing parameter sharing across heterogeneous learning signals.

\subsection{Feature Compositions Induce Distinct Learning Signals}
\label{sec:comp_grad}

In this subsection, we ask whether tokens with different feature compositions induce different parameter-update directions. We analyze this by comparing composition-conditioned parameter-gradient directions.

\textbf{Feature compositions.}
Let \(\mathcal{Y}_1,\ldots,\mathcal{Y}_m\) be the label spaces for \(m\) features decodable from \(h_t^{(\ell)}\).
We define the \emph{feature composition} of a token representation \(h_t^{(\ell)}(x)\) as the tuple
\begin{equation}
c\!\left(h_t^{(\ell)}(x)\right) = (y_1,\ldots,y_m), 
\qquad y_i \in \mathcal{Y}_i .
\label{eq:composition}
\end{equation}
Although the full product \(\mathcal{Y}_1\times\cdots\times\mathcal{Y}_m\) can be large, we work with the empirical subset observed in data.
In our experiments, \((y_1,\ldots,y_m)\) is instantiated using ground-truth labels when available (domain/stance) and bucketed statistics (frequency/position).

\textbf{Gradients as learning signals.}
Let \(\ell_{x,t}(\theta)=-\log p_\theta(x_{t+1}\mid x_{\le t})\) denote the token-level next-token loss.
For a parameter block \(\theta^{(\ell)}\) at layer \(\ell\) that is updated during continual training, we define the token-level parameter-gradient
\begin{equation}
g^{(\ell)}_{x,t}
=
\nabla_{\theta^{(\ell)}} \, \ell_{x,t}(\theta)
\in \mathbb{R}^{|\theta^{(\ell)}|}.
\label{eq:grad_on_theta}
\end{equation}
These gradients determine the update direction; comparing their \emph{directions} reveals whether different compositions push parameters in similar or different ways.

\textbf{Composition-conditioned mean directions.}
For a composition \(c\), let
\[
\mathcal{S}_c^{(\ell)}=\{(x,t)\,:\,c(h^{(\ell)}_t(x))=c\}
\]
be the set of tokens with composition \(c\) at layer \(\ell\).
We aggregate per-token gradients into a composition-conditioned mean \emph{direction}:
\begin{equation}
\bar{g}^{(\ell)}(c)
=
\frac{1}{|\mathcal{S}_c^{(\ell)}|}
\sum_{(x,t)\in\mathcal{S}_c^{(\ell)}}
\frac{g^{(\ell)}_{x,t}}{\|g^{(\ell)}_{x,t}\|_2+\varepsilon},
\label{eq:mean_dir}
\end{equation}
where \(\varepsilon\) is a small constant for numerical stability.
Normalizing each token gradient focuses this analysis on directional agreement (interference/alignment) rather than magnitude.
We compare compositions via cosine similarity
\begin{equation}
\mathrm{Sim}^{(\ell)}(c_1,c_2)
=
\cos\!\Big(\bar{g}^{(\ell)}(c_1),\bar{g}^{(\ell)}(c_2)\Big).
\label{eq:comp_cos}
\end{equation}

\textbf{Within-composition coherence vs.\ cross-composition weak alignment.}
We evaluate two notions of gradient-direction agreement:
(i) \textbf{within-composition coherence}: for each composition \(c\), randomly partition \(\mathcal{S}_c^{(\ell)}\) into two disjoint subsets \(A\) and \(B\), compute \(\bar g_A^{(\ell)}(c)\) and \(\bar g_B^{(\ell)}(c)\) via Eq.~\eqref{eq:mean_dir}, and measure \(\cos(\bar g_A^{(\ell)}(c),\bar g_B^{(\ell)}(c))\);
(ii) \textbf{cross-composition agreement}: sample distinct compositions \(c_1\neq c_2\) and compute \(\mathrm{Sim}^{(\ell)}(c_1,c_2)\).
Fig.~\ref{fig:grad_alignment} shows that within-composition directions are consistently aligned, while cross-composition similarities concentrate near zero.
Thus, different feature compositions induce \emph{diverse} learning signals: a single shared update direction cannot simultaneously align well with many compositions.
This observation motivates why \emph{composition mixing} within a single MoE route can be harmful.

\begin{figure}[h]
  \centering
  \includegraphics[width=0.9\linewidth]{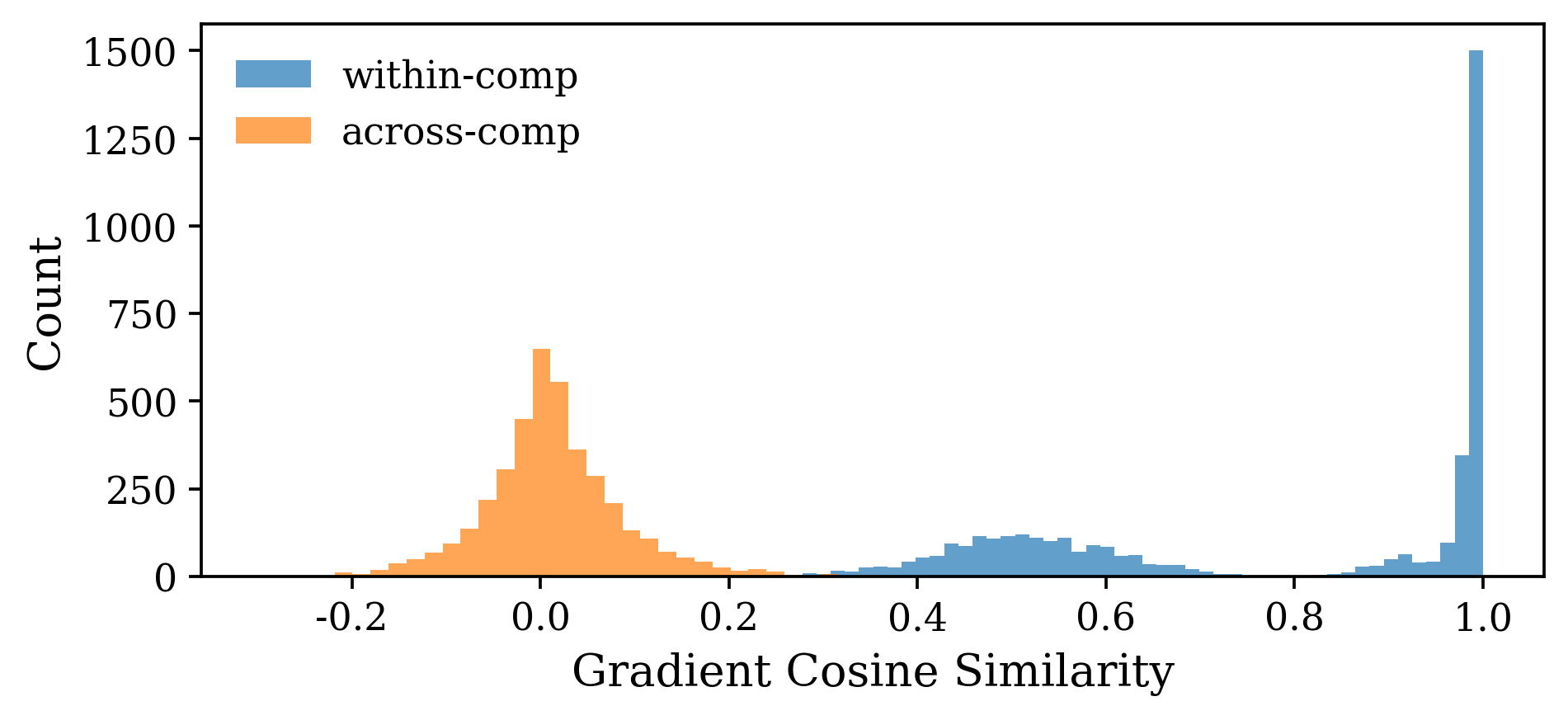}
  \caption{\textbf{Different feature compositions induce distinct gradient directions.}
Histogram of cosine similarity between composition-conditioned mean gradient directions (Eq.~\eqref{eq:mean_dir}--\eqref{eq:comp_cos}).
Splits of the same composition show high agreement, whereas different compositions concentrate near zero similarity, indicating weak alignment between their learning signals.}
  \label{fig:grad_alignment}
\end{figure}

These results establish that feature compositions correspond to stable, composition-specific update directions: gradients are coherent within a composition but weakly aligned across compositions. Consequently, when multiple compositions share parameters, their updates are more likely to interfere. This motivates the next subsection, where we quantify composition mixing within routes and connect it to route-level forgetting.

\subsection{From Composition Mixing to Forgetting in MoE Routing}
\label{sec:forgetting}

This subsection asks whether \emph{composition mixing} induced by standard MoE routing predicts catastrophic forgetting. We answer this in three steps: (i) define route-wise composition mixing under the old-task token distribution; (ii) define route-conditioned old-task loss and route-wise forgetting, and derive a theoretical link between higher mixing and greater susceptibility; (iii) empirically test how forgetting varies with mixing across routes while controlling for old-task exposure.

\textbf{Route assignment.}
For clarity, we present the analysis under top-1 routing; top-$k$ follows analogously.
Consider an MoE layer with \(K\) experts.
Given router input \(h^{(\ell)}_t(x)\in\mathbb{R}^d\), the router produces logits \(a_t^{(\ell)}(x)\in\mathbb{R}^K\) and selects
\begin{equation}
r_t^{(\ell)}(x) = \arg\max_{k\in[K]} a^{(\ell)}_{t,k}(x),
\label{eq:route_def}
\end{equation}
which we call the \emph{route} for token \((x,t)\) at layer \(\ell\).

\textbf{Route-wise composition mixing under old-task tokens.}
Let \(c(h_t^{(\ell)}(x))\) denote the feature composition (Eq.~\eqref{eq:composition}).
We define the distribution of compositions conditioned on route \(r\) \emph{under the old-task token distribution}:
\begin{equation}
p^{(\ell)}(c \mid r)
=
\Pr_{(x,t)\sim \mathcal{D}_{\mathrm{old}}}\!\Big[c(h_t^{(\ell)}(x))=c \,\Big|\, r_t^{(\ell)}(x)=r\Big].
\label{eq:pc_given_r}
\end{equation}
If routing were composition-selective, \(p^{(\ell)}(c\mid r)\) would concentrate on a small number of compositions.
A broad \(p^{(\ell)}(c\mid r)\) indicates \emph{composition mixing} within the route.
We focus on \(\mathcal{D}_{\mathrm{old}}\) because forgetting is evaluated on old-task tokens: higher mixing implies that a larger and more diverse fraction of old-task compositions share a route whose parameters will later be updated by new-task training, increasing the chance of interference.

\textbf{Effective composition number.}
We quantify mixing using the effective number of compositions:
\begin{equation}
N_{\mathrm{eff}}^{(\ell)}(r)
=
\frac{1}{\sum_{c} \left(p^{(\ell)}(c \mid r)\right)^2 }.
\label{eq:neff}
\end{equation}
This equals \(1\) if all tokens routed to \(r\) share the same composition, and increases as \(p^{(\ell)}(c\mid r)\) spreads.

\begin{figure}[h]
  \centering
  \includegraphics[width=0.9\linewidth]{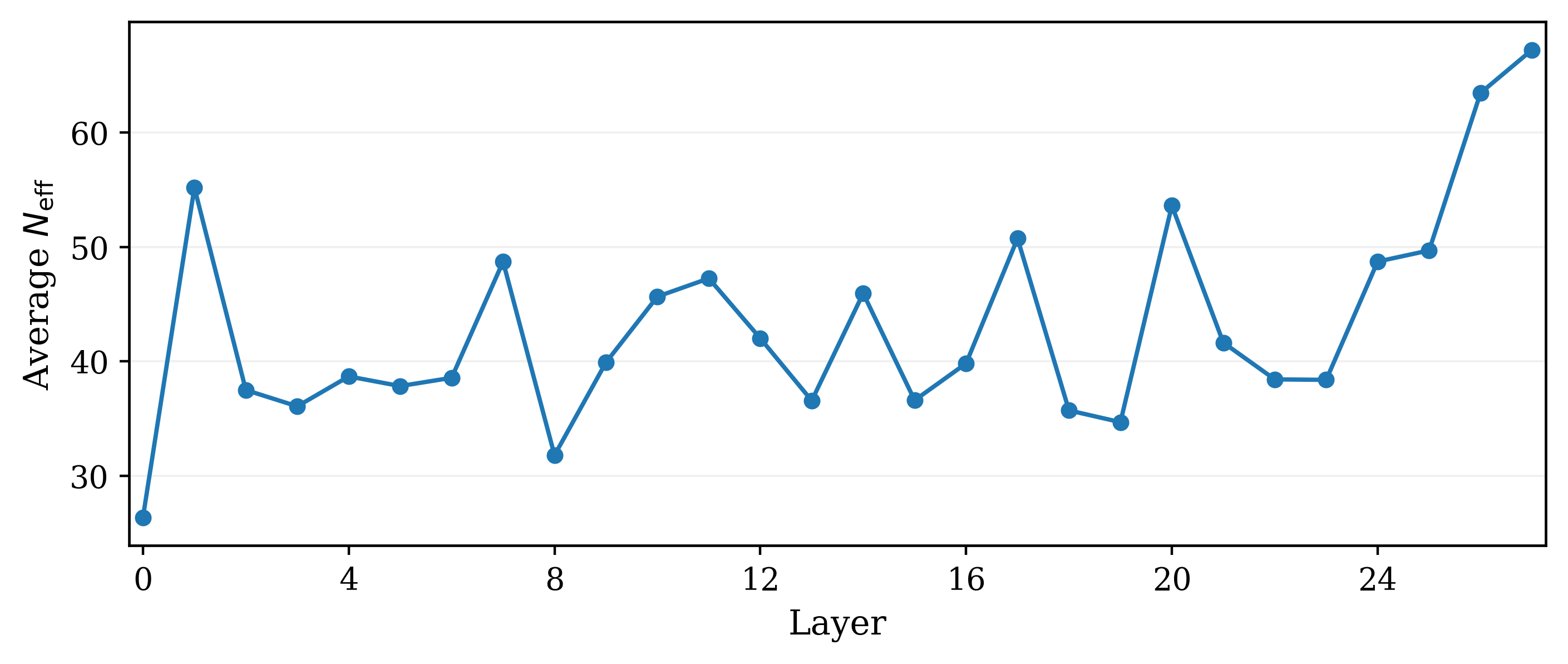}
  \caption{\textbf{Composition mixing persists across layers.}
Old-task mass-weighted average effective composition number \(N_{\mathrm{eff}}\) (Eq.~\eqref{eq:neff}) across routes in each MoE layer.
Values substantially above \(1\) indicate that routes typically aggregate multiple feature compositions under the old-task distribution.}
  \label{fig:layer_neff}
\end{figure}

Fig.~\ref{fig:layer_neff} shows that the old-task exposure-weighted \(N_{\mathrm{eff}}\) remains substantially above \(1\) throughout the MoE stack, indicating that routes typically aggregate multiple feature compositions under the old-task token distribution.

\textbf{Route-conditioned old-task loss and forgetting.}
For each MoE module at layer \(\ell\), we compute the route-conditioned loss on old-task data:
\begin{equation}
L_{\mathrm{old}}^{(\ell)}(r;\theta)
=
\mathbb{E}\!\left[\,-\log p_\theta(x_{t+1}\mid x_{\le t}) \;\middle|\; r^{(\ell)}_{x,t}=r \right],
\label{eq:route_old_loss}
\end{equation}
and define route-wise forgetting as
\begin{equation}
\Delta L_{\mathrm{old}}^{(\ell)}(r)
=
L_{\mathrm{old}}^{(\ell)}(r;\theta_{\mathrm{new}})
-
L_{\mathrm{old}}^{(\ell)}(r;\theta_{\mathrm{old}}).
\label{eq:deltaL_route}
\end{equation}

\textbf{Why higher mixing increases susceptibility.}
Our key mechanism is: if a route \(r\) mixes many distinct compositions, then no small subset of
``well-protected'' compositions can account for most old-task tokens routed to \(r\).
Consequently, a nontrivial fraction of the old-task mass must lie in compositions that are not reliably protected under typical update directions, making the route more susceptible to forgetting. This is consistent with Fig.~\ref{fig:grad_alignment}: since cross-composition gradient directions are weakly aligned, an update direction that preserves a small subset of compositions is unlikely to simultaneously preserve many others that share the same route.

We formalize this in two steps.
Lemma~\ref{lem:neff_mass} relates the effective composition number \(N_{\mathrm{eff}}(r)\)
to how much probability mass can be concentrated on any \(m\) compositions.
Theorem~\ref{thm:mixing_forgetting_compact} then converts this mass guarantee into a lower bound on the
\emph{route-level} old-loss increase.

\begin{lemma}[Mixing mass bound]
\label{lem:neff_mass}
Fix a route \(r\) with composition distribution \(p(c\mid r)\) over \(c\in\mathcal C\), and define
\[
N_{\mathrm{eff}}(r)\;=\;\Bigl(\sum_{c\in\mathcal C} p(c\mid r)^2\Bigr)^{-1}.
\]
Then for any \(S\subseteq\mathcal C\) with \(|S|\le m\),
\[
\Pr_{C\sim p(\cdot\mid r)}[C\notin S]\ \ge\ 1-\sqrt{\frac{m}{N_{\mathrm{eff}}(r)}}.
\]
\end{lemma}

\begin{theorem}[Composition mixing increases forgetting susceptibility]
\label{thm:mixing_forgetting_compact}
Fix a route \(r\). For each \(c\in\mathcal C\), let \(F_c(\theta)\) be the old-task loss restricted to tokens
routed to \(r\) with composition \(c\), and define
\[
F_r(\theta)\;=\;\mathbb{E}_{C\sim p(\cdot\mid r)}\!\big[F_C(\theta)\big].
\]
Consider one update \(\theta^+=\theta-\eta\hat u\) with \(\|\hat u\|=1\), \(\eta>0\).
Assume each \(F_c\) is \(L\)-smooth and \(\|\nabla F_c(\theta)\|\le G\) for all \(c\).

Let \(S\subseteq\mathcal C\) with \(|S|\le m\) denote a subset of compositions that happen to be well-aligned with the update. Suppose there exist \(\rho\in(0,1]\) and \(\kappa>0\) such that
for all \(c\notin S\),
\[
\Pr\!\big(F_c(\theta^+)-F_c(\theta)\ \ge\ \kappa\big)\ \ge\ \rho,
\]
where the probability is over the randomness defining \(\hat u\).
Let \(a:=\sqrt{m/N_{\mathrm{eff}}(r)}\), so that by Lemma~\ref{lem:neff_mass},
\(\Pr_{C\sim p(\cdot\mid r)}[C\notin S]\ge (1-a)_+\).
Then
\[
\mathbb{E}\!\big[F_r(\theta^+)-F_r(\theta)\big]
\ \ge\
(1-a)_+\Big(\rho\kappa-(1-\rho)\big(\eta G+\tfrac{L\eta^2}{2}\big)\Big).
\]
In particular, if \(\rho\kappa>(1-\rho)\big(\eta G+\tfrac{L\eta^2}{2}\big)\), the lower bound is nondecreasing in \(N_{\mathrm{eff}}(r)\) and becomes positive for sufficiently
large \(N_{\mathrm{eff}}(r)\).
\end{theorem}

\emph{Proofs in Appendix~\ref{app:mixing_forgetting_proof}.}

\textbf{Empirical association between mixing and forgetting.}
We empirically test the theorem's qualitative prediction by pooling routes across modules and examining how \(\Delta L_{\mathrm{old}}\) varies with the mixing score \(N_{\mathrm{eff}}\).
To ensure the trend reflects exposure of \emph{old-task tokens} (rather than being dominated by many rarely used routes), we form bins by \emph{mass-quantiles} of \(N_{\mathrm{eff}}\) under the old-task routing mass
\(\mathrm{mass}_{\mathrm{old}}(r)=\Pr_{(x,t)\sim\mathcal{D}_{\mathrm{old}}}[r^{(\ell)}_{x,t}=r]\),
so each bin contains approximately equal total old-task routing mass.
Within each bin, we report the mean (and standard error) of \(\Delta L_{\mathrm{old}}\) across routes.
Fig.~\ref{fig:neff_binplot} shows that routes with larger \(N_{\mathrm{eff}}\) are associated with greater \(\Delta L_{\mathrm{old}}\), consistent with the susceptibility mechanism highlighted by Theorem~\ref{thm:mixing_forgetting_compact}.

\begin{figure}[h]
  \centering
  \includegraphics[width=0.9\linewidth]{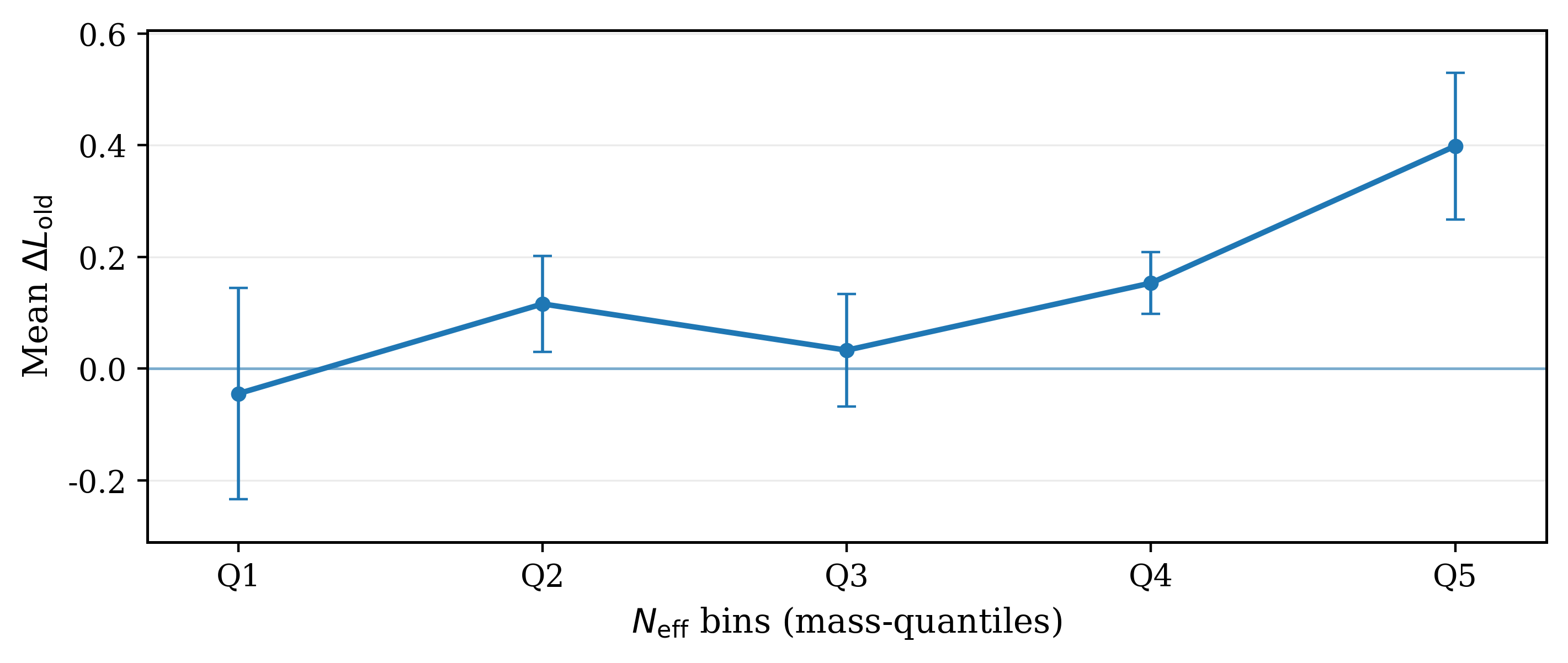}
  \caption{\textbf{More mixed routes forget more.}
Route-wise old-task loss increase \(\Delta L_{\mathrm{old}}\) versus effective composition number \(N_{\mathrm{eff}}\) (Eq.~\eqref{eq:neff}).
Routes are binned by mass-quantiles of \(N_{\mathrm{eff}}\) under old-task routing exposure \(\mathrm{mass}_{\mathrm{old}}(r)\), so each bin contains comparable old-task token mass.
Points report mean \(\Delta L_{\mathrm{old}}\) with standard error, showing a positive association between mixing and forgetting.}

  \label{fig:neff_binplot}
\end{figure}

%% file: section/method.tex
\section{Multi-Head Mixture-of-Experts (MH-MoE)}
\label{sec:method}

We introduce \textbf{MH-MoE}, a Transformer--MoE layer that performs expert routing
\emph{independently over multiple sub-representations} of the post-attention state.
Unlike standard MoE, which collapses the full representation into a single routing decision,
MH-MoE factorizes the representation into $H$ head-aligned slices and routes each slice separately,
yielding a tuple-valued routing decision. This increases routing resolution and reduces
feature-composition collisions, improving retention in continual learning.

\textbf{Head-aligned splitting.}
Let $h_t^{(\ell)}\in\mathbb{R}^d$ be the post-attention token representation at layer $\ell$
and position $t$. We partition the feature dimension into $H$ disjoint slices:
\begin{equation}
\label{eq:mhmoe_split}
\begin{aligned}
h_t^{(\ell)}
&=\big[\,h_{t,1}^{(\ell)} \,\Vert\, \cdots \,\Vert\, h_{t,H}^{(\ell)}\,\big],\\
h_{t,m}^{(\ell)} &\in \mathbb{R}^{d/H},\qquad m\in[H].
\end{aligned}
\end{equation}

\textbf{Head-private routing.}
Each head $m$ has its own router $W_{\mathrm{rt}}^{(\ell,m)}$ and its own private expert bank
$\{E^{(m)}_k\}_{k=1}^K$. The router selects the top-$k$ experts using only the head slice:
\begin{equation}
\label{eq:mhmoe_router}
\begin{aligned}
a_{t,m}^{(\ell)}
&= W_{\mathrm{rt}}^{(\ell,m)}\, h_{t,m}^{(\ell)} \in \mathbb{R}^{K},\\
\mathcal{S}_{t,m}^{(\ell)}
&=\mathrm{TopK}\!\big(a_{t,m}^{(\ell)},\,k\big)\subseteq [K],\\
\alpha_{t,m,j}^{(\ell)}
&=\frac{\exp\!\big(a_{t,m,j}^{(\ell)}\big)}
{\sum_{q\in \mathcal{S}_{t,m}^{(\ell)}} \exp\!\big(a_{t,m,q}^{(\ell)}\big)},
\qquad j\in \mathcal{S}_{t,m}^{(\ell)} .
\end{aligned}
\end{equation}
The overall routing decision is the tuple of selected expert sets
$\mathbf{S}_t^{(\ell)}=\big(\mathcal{S}_{t,1}^{(\ell)},\ldots,\mathcal{S}_{t,H}^{(\ell)}\big)$.

\textbf{Head-private expert output aggregation.}
Each expert takes a $d/H$-dimensional slice and produces a full $d$-dimensional output:
\begin{equation}
\label{eq:mhmoe_expert}
E^{(m)}_k:\mathbb{R}^{d/H}\rightarrow \mathbb{R}^{d}.
\end{equation}
Given $\mathbf{S}_t^{(\ell)}$, MH-MoE applies the selected experts for each head:
\begin{equation}
\label{eq:mhmoe_exec}
y_{t,m}^{(\ell)}
=\sum_{j\in \mathcal{S}_{t,m}^{(\ell)}}
\alpha_{t,m,j}^{(\ell)}\,
E^{(m)}_{j}\!\big(h_{t,m}^{(\ell)}\big)
\in \mathbb{R}^{d},
\qquad m\in[H],
\end{equation}
and aggregates by summation:
\begin{equation}
\label{eq:mhmoe_sum}
y_t^{(\ell)}=\sum_{m=1}^{H} y_{t,m}^{(\ell)} \in \mathbb{R}^{d}.
\end{equation}

\textbf{Implicit routing resolution.}
Although MH-MoE stores only $HK$ experts (organized into $H$ private banks),
the tuple $\mathbf{S}_t^{(\ell)}$ induces an implicit
$\big(\binom{K}{k}\big)^H$-way partition of tokens.
Equivalently, the selection-indexed composite mapping is
\begin{equation}
\label{eq:mhmoe_composite}
E_{\mathbf{S}}(h)
=\sum_{m=1}^{H}\;
\sum_{j\in \mathcal{S}_{m}}
\alpha_{m,j}\,E^{(m)}_{j}(h_m).
\end{equation}
where \(h=[h_1\Vert\cdots\Vert h_H],\;
\mathbf{S}=(\mathcal{S}_1,\ldots,\mathcal{S}_H).\)

\begin{algorithm}[tb]
  \caption{\textsc{MH-MoE}}
  \label{alg:mhmoe_forward}
  \begin{algorithmic}
    \STATE {\bfseries Input:} token states $\{h_t^{(\ell)}\}_{t=1}^{T}$, $h_t^{(\ell)}\!\in\!\mathbb{R}^d$; heads $H$ ($H\mid d$); routers $\{W_{\mathrm{rt}}^{(\ell,m)}\}_{m=1}^{H}$; experts $\{E^{(m)}_j\}_{m=1,j=1}^{H,K}$; top-$k$.
    \STATE {\bfseries Output:} $\{y_t^{(\ell)}\}_{t=1}^{T}$, $y_t^{(\ell)}\!\in\!\mathbb{R}^{d}$.
    \FOR{$t=1$ {\bfseries to} $T$}
      \STATE Split $h_t^{(\ell)}=[h_{t,1}^{(\ell)}\Vert\cdots\Vert h_{t,H}^{(\ell)}]$, $h_{t,m}^{(\ell)}\in\mathbb{R}^{d/H}$; set $y_t^{(\ell)}\leftarrow \mathbf{0}$.
      \FOR{$m=1$ {\bfseries to} $H$}
        \STATE $a_{t,m}^{(\ell)} \leftarrow W_{\mathrm{rt}}^{(\ell,m)} h_{t,m}^{(\ell)}$; $\mathcal{S}_{t,m}^{(\ell)} \leftarrow \mathrm{TopK}(a_{t,m}^{(\ell)},k)$.
        \STATE $\alpha_{t,m,\cdot}^{(\ell)} \leftarrow \mathrm{Softmax}\!\left(a_{t,m,\cdot}^{(\ell)}\right)$ restricted to $\mathcal{S}_{t,m}^{(\ell)}$.
        \STATE $y_t^{(\ell)} \leftarrow y_t^{(\ell)} + \sum_{j\in \mathcal{S}_{t,m}^{(\ell)}} \alpha_{t,m,j}^{(\ell)}\, E^{(m)}_{j}\!\left(h_{t,m}^{(\ell)}\right)$.
      \ENDFOR
    \ENDFOR
  \end{algorithmic}
\end{algorithm}

%% file: section/related_work.tex
\section{Related Work}
\textbf{Mixture of Experts} Mixture of Experts (MoE) is a foundational paradigm for conditional computation, where a router dynamically selects among multiple experts so that different subsets of parameters specialize for different inputs\cite{jacobs1991adaptive,jordan1994hierarchical}. Early work extended MoE beyond shallow mixtures to deep architectures with stacked routers and experts to increase capacity and expressivity\cite{eigen2013learning}. A major practical breakthrough was the sparse MoE layer\cite{shazeer2017outrageously}, which enforces sparse expert activation per example/token, improving scalability and training stability while reducing compute. Since then, MoE has been integrated into a wide range of neural backbones—including convolutional and Transformer-based models—and has achieved strong results across tasks. In the LLM regime, MoE is widely adopted to scale model capacity under fixed compute budgets, and substantial effort has been devoted to routing design\cite{lepikhin2020gshard,fedus2022switch,du2022glam}. Representative strategies include token-driven routing where each token selects its top-(k) experts\cite{shazeer2017outrageously,fedus2022switch}, expert-driven routing where experts select the top-(k) tokens to process\cite{zhou2022mixture}, and global assignment schemes that decide expert allocation at a higher granularity\cite{lewis2021base,riquelme2021scaling}.

\textbf{Catastrophic Forgetting} Catastrophic forgetting refers to the rapid degradation of previously learned capabilities when a model is trained sequentially on new data, a phenomenon classically attributed to \emph{parameter interference} in shared networks where updates for new tasks overwrite weights supporting earlier tasks\cite{mccloskey1989catastrophic}. Mitigation strategies broadly fall into (i) \emph{regularization/importance-based} methods that constrain changes to parameters deemed crucial for past tasks (e.g., EWC \cite{kirkpatrick2017overcoming}and Synaptic Intelligence\cite{zenke2017continual}s) , (ii) \emph{replay and constraint-based} methods that preserve past behavior by episodic memory or gradient projection\cite{lopez2017gradient,rebuffi2017icarl} , and (iii) \emph{architectural isolation/expansion} approaches that allocate different tasks to different parameter subsets or grow capacity over time to reduce interference. In the LLM setting, recent empirical studies confirm that continual instruction tuning and sequential domain adaptation can induce substantial forgetting across knowledge and reasoning abilities\cite{luo2025empirical,ke2023continual}, motivating a growing literature that revisits continual learning principles under the scale and representation entanglement of modern LLMs\cite{shi2025continual}.

%% file: section/experiment.tex
\section{Experiments}
\label{sec:experiment}
\input{table/main_table}

\subsection{Experimental Setup}
We evaluate whether MH-MoE improves continual learning by reducing forgetting while maintaining strong overall accuracy.

\textbf{Benchmark.}
We use TRACE~\cite{wang2023trace}, a continual-learning suite of eight diverse tasks: C-STANCE, FOMC, MeetingBank, Py150, ScienceQA, NumGLUE-cm, NumGLUE-ds, and 20Minuten.
We follow the standard TRACE protocol and train sequentially over tasks.

\textbf{Base models.}
We build MH-MoE on pretrained Qwen3-0.6B and Qwen3-8B~\cite{yang2025qwen3}.
Unless otherwise noted, all methods share the same tokenizer, maximum sequence length, optimizer family, and training schedule.

\textbf{Baselines.}
We compare against (i) SeqLoRA, which sequentially trains a single shared LoRA adapter across tasks; (ii) LoRAMoE~\cite{dou2024loramoe}, a single-representation-routed MoE matched to MH-MoE in activated parameters per token; and (iii) continual-learning baselines EWC~\cite{kirkpatrick2017overcoming}, GEM~\cite{lopez2017gradient}, and O-LoRA~\cite{wang2023orthogonal}.

\textbf{Hardware.}
All experiments run on a cluster of 64 NVIDIA H100 GPUs.

\textbf{Metrics.}
We report \textbf{Overall Performance (OP)}, the average score across all tasks after the final task, and \textbf{Backward Transfer (BWT)}, the mean change on earlier tasks after learning later ones. Higher OP and higher (less negative) BWT indicate better continual-learning performance.

Experimental details are provided in Appendix~\ref{app:exp_detail}.

\subsection{Main Results}
\label{sec:main-results}

\textbf{MH-MoE achieves the best overall forgetting--accuracy tradeoff among routing-based methods.}
Table~\ref{tab:trace_final_per_task} shows that MH-MoE consistently improves over LoRAMoE, in both overall performance (OP) and backward transfer (BWT).
On Qwen3-0.6B, MH-MoE raises OP from 0.378 to 0.467 and substantially reduces forgetting (BWT from $-0.112$ to $-0.045$), with broad gains across datasets (notably C-STANCE and FOMC).
On Qwen3-8B, where forgetting is already milder, MH-MoE still attains the strongest aggregate results, improving OP from 0.551 to 0.569 and achieving the least-negative BWT ($-0.051$ vs.\ $-0.055$), indicating that the benefits of multi-head routing persist at larger model scales.

\textbf{MH-MoE is competitive with task-aware continual-learning baselines while remaining task-agnostic.}
Against regularization- and rehearsal-based methods (EWC, GEM), MH-MoE achieves higher OP and less negative BWT on both backbones
(Qwen3-0.6B: OP 0.467 vs.\ 0.355/0.357; BWT $-0.045$ vs.\ $-0.123/-0.124$).
O-LoRA is strong on some tasks via explicit orthogonality constraints, but MH-MoE offers a better overall retention--accuracy tradeoff (Qwen3-0.6B: BWT $-0.045$ vs.\ $-0.081$).
Unlike these baselines, which rely on explicit task boundaries during training, MH-MoE can be trained and deployed in a task-agnostic stream via token-level routing.

\subsection{Analysis}
\label{sec:exp_analysis}

\textbf{Head-private routing improves retention beyond route-space size and capacity.}
MH-MoE could reduce forgetting due to (i) a larger routing-outcome space or (ii) better \emph{feature isolation} from head-private routing.
To test whether (ii) matters \emph{independently} of (i), we construct a controlled comparison where the routing-outcome space and parameter count are kept comparable.
Specifically, we compare MH-MoE with $M{=}8$ heads and per-head top-1 routing among $4$ head-private experts (yielding $4^8{=}65{,}536$ routing paths) against \textsc{LoRAMoE} with $K{=}26$ experts and top-$5$ routing (yielding $\binom{26}{5}{=}65{,}780$ possible expert sets).
We also match the overall parameter budget between the two models.
Under this setup, the dominant architectural difference is whether routing is performed \emph{head-wise on sub-representations} (MH-MoE) or \emph{globally on a head-mixed representation} (\textsc{LoRAMoE}). Table~\ref{tab:route_space} shows that MH-MoE yields substantially better retention on TRACE (higher BWT) despite comparable route-space cardinality and capacity.
This rules out the trivial explanation that MH-MoE works merely by expanding the number of routing outcomes or increasing model size, and instead supports our hypothesis that \emph{head-private routing itself} reduces destructive parameter sharing.

To connect this improvement to our analysis, we further compute the route/path-wise mixing score $N_{\mathrm{eff}}$ using the same factor set (domain, stance, frequency, position).
As shown in Fig. \ref{fig:ablate_route}, even under matched route-space size, MH-MoE exhibits lower $N_{\mathrm{eff}}$ than \textsc{LoRAMoE}, indicating fewer semantic compositions collide within the same update destination.
Together with Section~\ref{sec:comp_grad}, which shows that different compositions induce weakly aligned gradient directions, this offers a mechanistic account: head-private routing reduces composition mixing (lower $N_{\mathrm{eff}}$), thereby reducing gradient conflict and improving retention.

\textbf{Smaller route spaces exacerbate composition collision.}
We additionally consider a more constrained \textsc{LoRAMoE} baseline with $K{=}4$ experts and top-1 routing (reported in Table~\ref{tab:trace_final_per_task}).
Its routing-outcome space is orders of magnitude smaller than the MH-MoE configuration analyzed above, so many more compositions must share the same update destination.
Consistent with this prediction, the measured $N_{\mathrm{eff}}$ is larger (more mixing), providing further evidence that $N_{\mathrm{eff}}$ faithfully captures composition collision and tracks retention behavior across routing granularities (Fig. \ref{fig:ablate_neff}).

\input{table/ablate_route}

\begin{figure}[ht]
  \centering

  \begin{subfigure}[t]{0.8\linewidth}
    \centering
    \includegraphics[width=\linewidth]{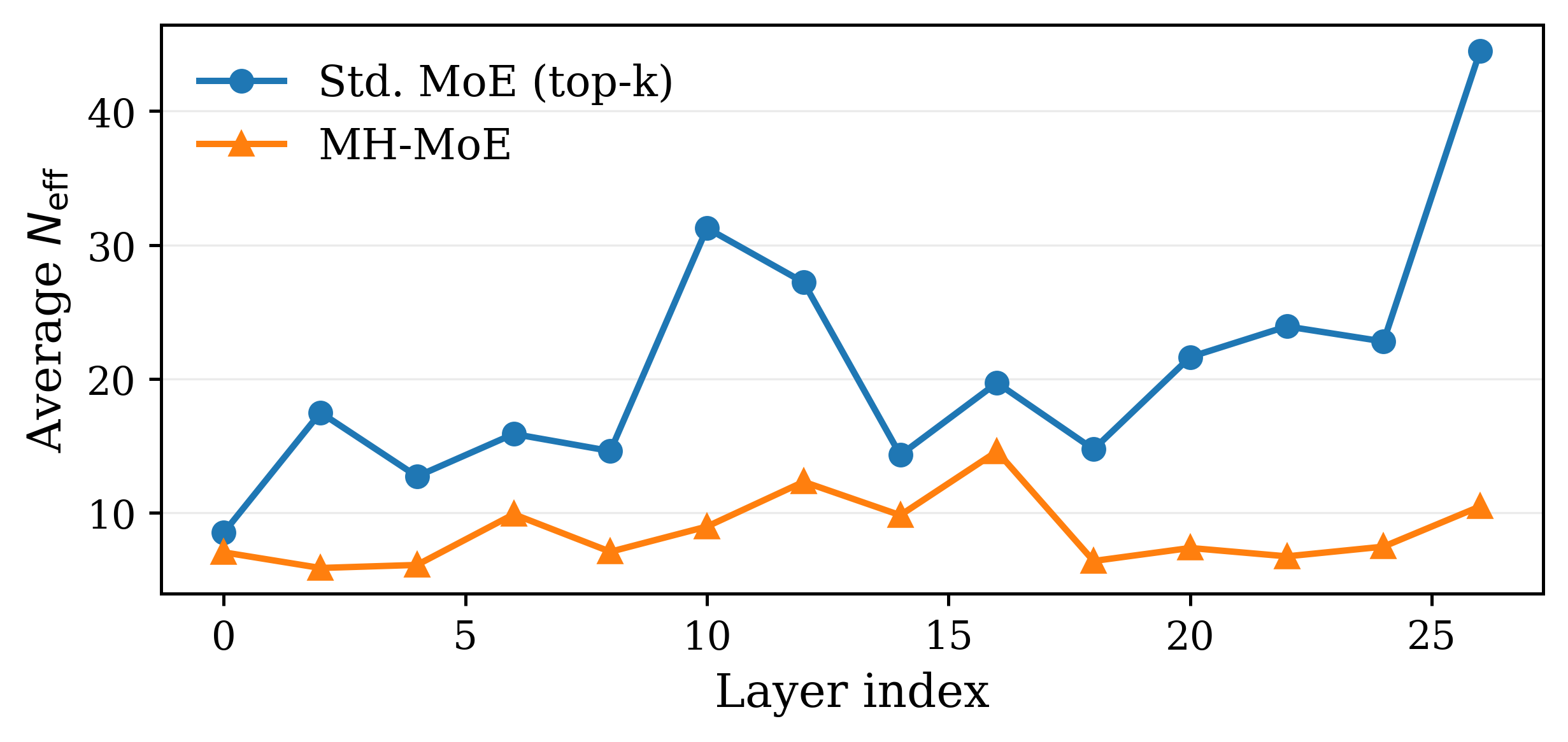}
    \caption{}
    \label{fig:ablate_route}
  \end{subfigure}
  \begin{subfigure}[t]{0.8\linewidth}
    \centering
    \includegraphics[width=\linewidth]{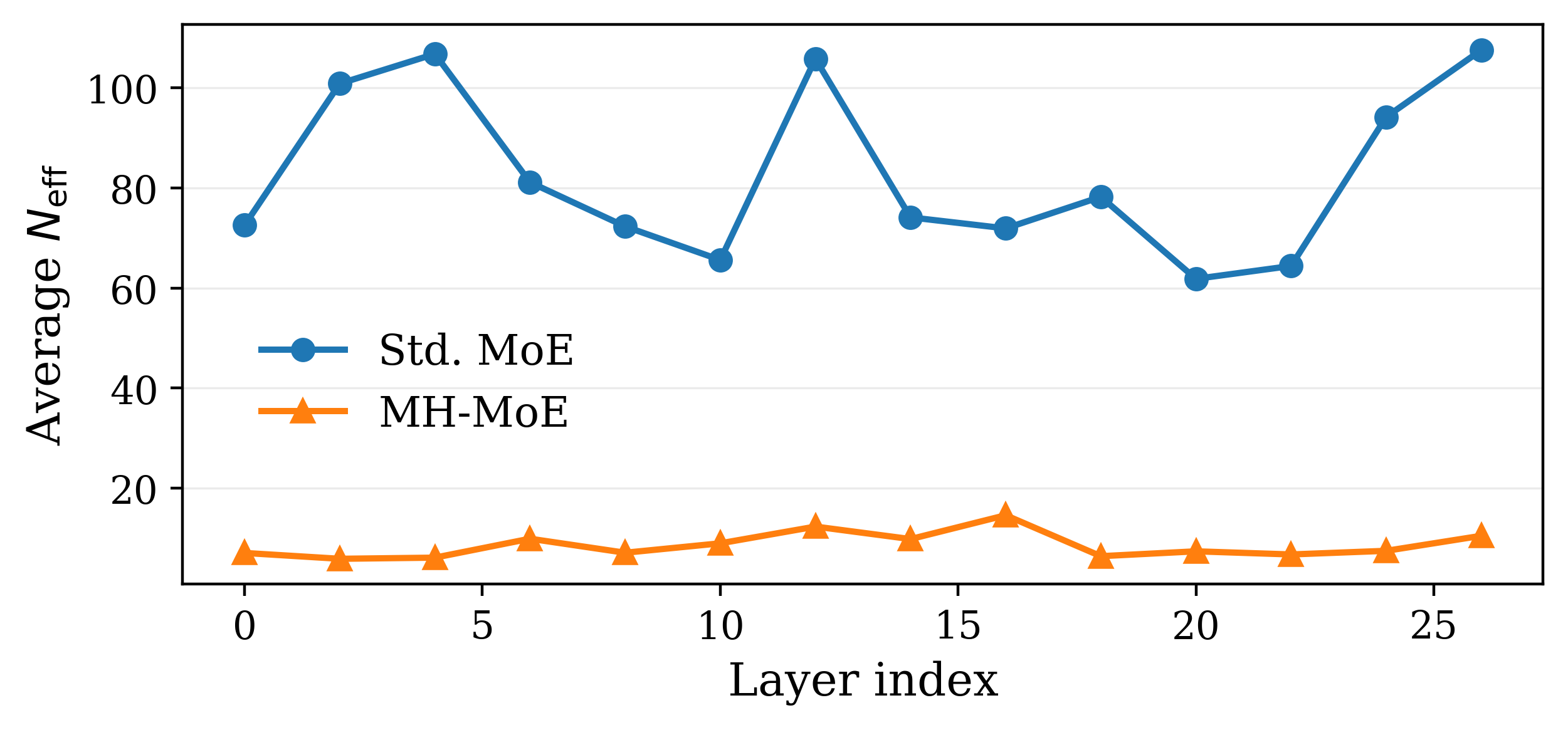}
    \caption{}
    \label{fig:ablate_neff}
  \end{subfigure}

  \caption{\textbf{Head-private routing reduces composition collision.}
\textbf{(a)} Under matched route-space size and activated budget (MH-MoE: $M{=}8$, top-1 over $4$ head-private experts; \textsc{LoRAMoE}: $K{=}26$, top-$5$), MH-MoE yields lower route/path-wise mixing $N_{\mathrm{eff}}$, indicating fewer semantic composition collisions.
\textbf{(b)} With a much smaller route space (\textsc{LoRAMoE}: $K{=}4$, top-1; Table~\ref{tab:trace_final_per_task}), $N_{\mathrm{eff}}$ increases, showing that constrained routing forces more compositions to share update destinations.}

  \label{fig:ablate_main}
\end{figure}

\textbf{MH-MoE mitigates forgetting consistently across task orderings.}
Continual-learning performance can be sensitive to the task sequence, so we additionally test whether MH-MoE’s retention gains persist under different TRACE orderings.
Table~\ref{tab:task_order} reports results for three task permutations.
Across all orders, MH-MoE achieves higher OP and substantially less negative BWT than \textsc{LoRAMoE}, indicating that the forgetting mitigation is not an artifact of a particular curriculum.
This robustness is consistent with our mechanism: head-private routing reduces composition collision under diverse input streams, yielding more stable retention regardless of task order.

\input{table/task_order}

\textbf{More heads increase routing granularity and improve overall performance.}
We ablate the number of routing heads $M$ in MH-MoE.
Increasing $M$ increases tuple-valued routing resolution (more virtual paths), which can reduce composition mixing, at the cost of additional routing/dispatch and reduced per-head capacity.
Keeping other settings fixed, we evaluate $M\in\{2,4,8,16\}$ on TRACE.
Table~\ref{tab:ablate_heads} shows that OP improves with $M$ and is best at $M{=}16$.
\input{table/ablation_head}

\textbf{Computation overhead.}
Table~\ref{tab:overhead} reports training overhead on Qwen3-8B with the base model frozen (batch size $B\!=\!1$, sequence length $T\!=\!512$, bf16).
MH-MoE with $M{=}8$ incurs only a small overhead relative to \textsc{LoRAMoE}: $3184.7$ tok/s vs.\ $3215.4$ tok/s,
$160.85$ ms vs.\ $158.99$ ms per training step (about $1.01\times$ higher latency), and $\sim$4\% higher peak memory,
consistent with additional head-wise routing and dispatch.

\input{table/overhead}

\section{Conclusion}
\label{sec:conclusion}

We study why MoE Transformers can still catastrophically forget under continual learning. We argue the bottleneck arises \emph{before} routing: multi-head attention mixes head-structured signals into a single router input, so routing is driven by feature co-occurrences rather than separable head channels. This yields composition collisions, which we quantify via the route-wise effective composition number \(N_{\mathrm{eff}}\) and show predicts route-local forgetting. Motivated by this, we propose \textsc{MH-MoE}, which routes head-wise sub-representations to reduce mixing and improves retention on TRACE over LoRAMoE and other continual-learning baselines.

%% file: table/main_table.tex
\begin{table*}[!htbp]
\centering
\caption{Continual learning performance on TRACE after training on all tasks.
We report final score on each dataset, Overall Performance (OP), and Backward Transfer (BWT).
Abbreviations: CS=C-STANCE, FM=FOMC, MB=MeetingBank, PY=Py150, SQ=ScienceQA, NC=NumGLUE-cm, ND=NumGLUE-ds, 20M=20Minuten.}
\label{tab:trace_final_per_task}
\vspace{0.3em}
\setlength{\tabcolsep}{3.4pt}
\renewcommand{\arraystretch}{1.02}
\small
\begin{tabular}{llcccccccccc}
\toprule
Base model & Method & CS & FM & MB & PY & SQ & NC & ND & 20M & OP$\uparrow$ & BWT$\uparrow$ \\
\midrule
\multirow{6}{*}{Qwen3-0.6B}
& SeqLoRA         &  12.8&  26.2&  14.6&  43.5&  59.8&  7.4&  34.8&  36.6&  29.5& -19.0 \\
& EWC                 &  46.6&  42.3&  17.6&  45.8&  48.2&  8.6&  37.5&  37.0&  35.5&  -12.3\\
& GEM                 &  46.1&  30.8&  17.7&  45.5&  53.2&  19.8&  35.7&  36.9&  35.7&  -12.4\\
& O-LoRA &  46.3&  54.4&  \textbf{20.9}&  47.9&  65.3&  27.2&  \textbf{48.3}&  \textbf{37.2}&  43.4&  -8.1\\
& LoRAMoE &  34.3&  52.0&  12.0&  46.3&  61.4&  14.8&  44.3&  \textbf{37.2}&  37.8&  -11.2\\
& MH-MoE       &  \textbf{48.8}&  \textbf{67.3}&  18.8&  \textbf{49.3}&  \textbf{70.1}&  \textbf{34.6}&  \textbf{48.3}&  36.7&  \textbf{46.7}&  \textbf{-4.5}\\
\midrule
\multirow{6}{*}{Qwen3-8B}
& SeqLoRA         &  49.6&  56.7&  16.7&  49.7&  88.8&  66.7&  63.7&  40.2&  54.0& -7.3 \\
& EWC                 &  50.1&  61.3&  17.1&  52.2&  86.5&  64.2&  64.3&  40.6&  54.5&  -7.5\\
& GEM                 &  \textbf{52.3}&  62.9&  18.0&  52.5&  86.2&  59.3&  61.8&  40.5&  54.3&  -6.4\\
& O-LoRA                 &  51.5&  64.5&  \textbf{21.5}&  \textbf{54.6}&  90.7&  66.7&  63.4&  40.6&  56.7&  -5.4\\
& LoRAMoE &  49.1&  65.7&  14.6&  52.8&  89.2&  \textbf{71.6}&  58.8&  39.2&  55.1&  -5.5\\
& MH-MoE       &  50.5&  \textbf{65.9}&  18.2&  50.9&  \textbf{90.8}&  67.9&  \textbf{70.5}&  \textbf{40.7}&  \textbf{56.9}& \textbf{-5.1}\\

\bottomrule
\end{tabular}
\end{table*}

%% file: table/ablate_route.tex
\begin{table}[!htbp]
\centering
\caption{Routing-strategy ablation.}
\label{tab:route_space}
\vspace{-0.2em}
\setlength{\tabcolsep}{2.8pt}
\renewcommand{\arraystretch}{0.98}
\scriptsize
\begin{tabular}{lcccccccccc}
\toprule
Method & CS & FM & MB & PY & SQ & NC & ND & 20M & OP$\uparrow$ & BWT$\uparrow$ \\
\midrule
LoRAMoE & \textbf{49.2} & 40.5 & 15.3 & 48.8 & 69.4 & 30.9 & \textbf{52.6} & \textbf{37.8} & 43.1 & -9.3 \\
MH-MoE  & 48.8 & \textbf{67.3} & \textbf{18.8} & \textbf{49.3} & \textbf{70.1} & \textbf{34.6} & 48.3 & 36.7 & \textbf{46.7} & \textbf{-4.5} \\
\bottomrule
\end{tabular}
\vspace{-0.6em}
\end{table}

%% file: table/task_order.tex
\begin{table}[!htbp]
\centering
\caption{Task-ordering ablation on TRACE.}
\label{tab:task_order}
\vspace{-0.2em}
\setlength{\tabcolsep}{2.6pt}
\renewcommand{\arraystretch}{0.98}
\scriptsize
\begin{tabular}{llcccccccccc}
\toprule
Order & Method & CS & FM & MB & PY & SQ & NC & ND & 20M & OP$\uparrow$ & BWT$\uparrow$ \\
\midrule
\multirow{2}{*}{1} & LoRAMoE & 20.8 & 0.6 & 16.2 & 41.4 & 36.0 & 32.0 & 42.8 & 37.0 & 28.4 & -21.4 \\
                  & MH-MoE  & \textbf{51.3} & \textbf{47.4} & \textbf{20.8} & \textbf{49.9} & \textbf{61.1} & \textbf{32.1} & \textbf{47.1} & \textbf{37.7} & \textbf{43.4} & \textbf{-8.0} \\
\midrule
\multirow{2}{*}{2} & LoRAMoE & 34.1 & 0.2 & 17.0 & 46.4 & 49.3 & \textbf{21.0} & \textbf{30.2} & 37.2 & 29.4 & -18.3 \\
                  & MH-MoE  & \textbf{49.2} & \textbf{55.6} & \textbf{20.6} & \textbf{50.0} & \textbf{57.4} & \textbf{21.0} & 26.5 & \textbf{37.6} & \textbf{39.7} & \textbf{-9.6} \\
\midrule
\multirow{2}{*}{3} & LoRAMoE & 29.5 & 12.9 & 26.0 & 45.6 & 50.3 & 21.0 & 15.4 & \textbf{33.7} & 29.3 & -17.0 \\
                  & MH-MoE  & \textbf{48.3} & \textbf{52.6} & \textbf{29.4} & \textbf{48.5} & \textbf{51.5} & \textbf{23.5} & \textbf{29.5} & 32.3 & \textbf{39.5} & \textbf{-10.9} \\
\bottomrule
\end{tabular}
\vspace{-0.6em}
\end{table}

%% file: table/ablation_head.tex
\begin{table}[!htbp]
\centering
\caption{MH-MoE head-count ablation on TRACE (OP$\uparrow$).}
\label{tab:ablate_heads}
\vspace{-0.2em}
\setlength{\tabcolsep}{6.5pt}
\renewcommand{\arraystretch}{0.98}
\scriptsize
\begin{tabular}{ccccc}
\toprule
$M$ & 2 & 4 & 8 & 16 \\
\midrule
OP$\uparrow$ & 53.1 & 52.6 & 54.5 & \textbf{56.9} \\
\bottomrule
\end{tabular}
\vspace{-0.6em}
\end{table}

%% file: table/overhead.tex
\begin{table}[!htbp]
\centering
\caption{Compute and memory overhead}
\label{tab:overhead}
\vspace{-0.2em}
\setlength{\tabcolsep}{2.4pt}
\renewcommand{\arraystretch}{0.98}
\scriptsize
\begin{tabular}{lccc}
\toprule
Method & Tok/s$\uparrow$ & Mem (GiB)$\downarrow$ & ms/step$\downarrow$ \\
\midrule
LoRAMoE & 3215.4 & 40.22 & 158.99 \\
MH-MoE  & 3184.7 & 41.95 & 160.85 \\
\bottomrule
\end{tabular}
\vspace{-0.6em}
\end{table}

%% file: section/appendix.tex
\section{Appendix}
\subsection{Head-wise Causal Importance}
\label{app:head_feature_metric}

Fix a layer \(\ell\). Let the router input at token \(t\) be \(h_t^{(\ell)}\in\mathbb{R}^{d}\),
with \(H\) heads and head dimension \(d_h\) so \(d=H d_h\). We write
\[
h_t^{(\ell)}=\big[h_{t,1}^{(\ell)};\dots;h_{t,H}^{(\ell)}\big],\qquad h_{t,m}^{(\ell)}\in\mathbb{R}^{d_h}.
\]
For each feature \(Y\), we train a linear probe \(g_Y^{(\ell)}:\mathbb{R}^{d}\to\mathcal{Y}\) on
\(\{(h_i,y_i)\}_{i=1}^{N}\) and evaluate it with \(\mathrm{Perf}(\cdot)\) (accuracy in our experiments).

\paragraph{Head mean-replacement ablation.}
Let \(\mu_m^{(\ell)}\in\mathbb{R}^{d_h}\) be the empirical mean of head \(m\)'s block over the probe dataset:
\[
\mu_m^{(\ell)} \;=\; \frac{1}{N}\sum_{i=1}^{N} h_{i,m}^{(\ell)}.
\]
Define \(\mathcal{A}_m:\mathbb{R}^{d}\to\mathbb{R}^{d}\) as replacing head \(m\)'s block by \(\mu_m^{(\ell)}\):
\[
\mathcal{A}_m(h)=\big[h_{1};\dots;h_{m-1};\mu_m^{(\ell)};h_{m+1};\dots;h_{H}\big].
\]

\paragraph{Importance score and shares.}
We define head \(m\)'s causal importance for feature \(Y\) at layer \(\ell\) as the probe performance drop:
\[
I_{Y,m}^{(\ell)}
=
\mathrm{Perf}\!\left(g_Y^{(\ell)}\right)
-
\mathrm{Perf}\!\left(g_Y^{(\ell)}\circ \mathcal{A}_m\right).
\]
We normalize across heads to obtain a per-feature distribution over heads:
\[
S_{Y,m}^{(\ell)}
=
\frac{I_{Y,m}^{(\ell)}}{\sum_{j=1}^{H} I_{Y,j}^{(\ell)} + \varepsilon},
\qquad \sum_{m=1}^{H} S_{Y,m}^{(\ell)}\approx 1,
\]
with \(\varepsilon\) a small constant for numerical stability.

\subsection{Proofs for Lemma~\ref{lem:neff_mass} and Theorem~\ref{thm:mixing_forgetting_compact}}
\label{app:mixing_forgetting_proof}

\begin{proof}[Proof of Lemma~\ref{lem:neff_mass}]
Let \(p(\cdot\mid r)\) be the distribution on \(\mathcal C\), and let \(S\subseteq\mathcal C\) with \(|S|\le m\).
By Cauchy--Schwarz,
\[
\sum_{c\in S} p(c\mid r)
\;=\;
\langle \mathbf{1}_S,\; p(\cdot\mid r)\rangle
\;\le\;
\|\mathbf{1}_S\|_2\;\|p(\cdot\mid r)\|_2
\;=\;
\sqrt{|S|}\;\sqrt{\sum_{c\in\mathcal C} p(c\mid r)^2}
\;\le\;
\sqrt{m}\;\sqrt{\sum_{c\in\mathcal C} p(c\mid r)^2}.
\]
Using \(N_{\mathrm{eff}}(r)=\big(\sum_{c\in\mathcal C} p(c\mid r)^2\big)^{-1}\), we obtain
\[
\Pr_{C\sim p(\cdot\mid r)}[C\in S]
=\sum_{c\in S}p(c\mid r)
\;\le\;
\sqrt{\frac{m}{N_{\mathrm{eff}}(r)}}.
\]
Therefore,
\[
\Pr_{C\sim p(\cdot\mid r)}[C\notin S]
\;=\;
1-\Pr_{C\sim p(\cdot\mid r)}[C\in S]
\;\ge\;
1-\sqrt{\frac{m}{N_{\mathrm{eff}}(r)}}.
\]
\end{proof}

\begin{proof}[Proof of Theorem~\ref{thm:mixing_forgetting_compact}]
Fix route \(r\), step size \(\eta>0\), and a (possibly random) unit update direction \(\hat u\) with
\(\theta^+=\theta-\eta\hat u\).
For each \(c\in\mathcal C\), define the one-step change
\(\Delta_c := F_c(\theta^+)-F_c(\theta)\).

\paragraph{Step 1: a uniform lower bound from smoothness and bounded gradients.}
Since each \(F_c\) is \(L\)-smooth, for any vector \(v\) we have the quadratic lower bound
\[
F_c(\theta+v)\ \ge\ F_c(\theta)+\langle \nabla F_c(\theta),v\rangle-\frac{L}{2}\|v\|_2^2.
\]
Applying this with \(v=-\eta\hat u\) yields
\[
F_c(\theta-\eta\hat u)
\;\ge\;
F_c(\theta)-\eta\langle \nabla F_c(\theta),\hat u\rangle-\frac{L\eta^2}{2}.
\]
Thus,
\[
\Delta_c
\;\ge\;
-\eta\langle \nabla F_c(\theta),\hat u\rangle-\frac{L\eta^2}{2}.
\]
Using \(\|\hat u\|=1\) and the assumption \(\|\nabla F_c(\theta)\|\le G\),
\[
-\langle \nabla F_c(\theta),\hat u\rangle
\;\ge\;
-\|\nabla F_c(\theta)\|\,\|\hat u\|
\;\ge\;
-G,
\]
hence for every \(c\in\mathcal C\),
\begin{equation}
\label{eq:delta_uniform_lb}
\Delta_c \;\ge\; -\eta G-\frac{L\eta^2}{2}.
\end{equation}

\paragraph{Step 2: expected per-composition lower bound for \(c\notin S\).}
Let \(c\notin S\). By assumption,
\(\Pr(\Delta_c\ge \kappa)\ge \rho\),
where the probability is over the randomness defining \(\hat u\).
Decomposing by events and using \eqref{eq:delta_uniform_lb},
\[
\mathbb{E}[\Delta_c]
=
\mathbb{E}[\Delta_c\,\mathbf{1}\{\Delta_c\ge \kappa\}]
+\mathbb{E}[\Delta_c\,\mathbf{1}\{\Delta_c< \kappa\}]
\;\ge\;
\kappa\,\Pr(\Delta_c\ge \kappa)
+\Big(-\eta G-\frac{L\eta^2}{2}\Big)\Pr(\Delta_c<\kappa).
\]
Therefore,
\begin{equation}
\label{eq:delta_c_outsideS}
\mathbb{E}[\Delta_c]
\;\ge\;
\rho\kappa-(1-\rho)\Big(\eta G+\frac{L\eta^2}{2}\Big),
\qquad \forall\,c\notin S.
\end{equation}

\paragraph{Step 3: lift to the route mixture using Lemma~\ref{lem:neff_mass}.}
By definition,
\[
F_r(\theta)=\mathbb{E}_{C\sim p(\cdot\mid r)}[F_C(\theta)]
\quad\Rightarrow\quad
F_r(\theta^+)-F_r(\theta)=\mathbb{E}_{C\sim p(\cdot\mid r)}[\Delta_C].
\]
We take expectation over the randomness defining \(\hat u\).
Since \(C\sim p(\cdot\mid r)\) is drawn from the (fixed) old-task mixture on route \(r\), it is independent of
the new-task randomness that defines \(\hat u\); hence we may apply iterated expectation:
\[
\mathbb{E}[F_r(\theta^+)-F_r(\theta)]
=
\mathbb{E}_{C\sim p(\cdot\mid r)}\big[\mathbb{E}[\Delta_C\mid C]\big]
=
\sum_{c\in\mathcal C} p(c\mid r)\,\mathbb{E}[\Delta_c].
\]
Dropping the contributions from \(c\in S\) and lower bounding the remainder by the worst case over \(c\notin S\),
\[
\mathbb{E}[F_r(\theta^+)-F_r(\theta)]
\;\ge\;
\sum_{c\notin S} p(c\mid r)\,\inf_{c\notin S}\mathbb{E}[\Delta_c]
=
\Pr_{C\sim p(\cdot\mid r)}[C\notin S]\cdot \inf_{c\notin S}\mathbb{E}[\Delta_c].
\]
Combining with \eqref{eq:delta_c_outsideS} yields
\[
\mathbb{E}[F_r(\theta^+)-F_r(\theta)]
\;\ge\;
\Pr_{C\sim p(\cdot\mid r)}[C\notin S]\,
\Big(\rho\kappa-(1-\rho)\big(\eta G+\tfrac{L\eta^2}{2}\big)\Big).
\]
Let \(a:=\sqrt{m/N_{\mathrm{eff}}(r)}\).
If \(a\ge 1\) then \((1-a)_+=0\) and the bound below is trivial.
Otherwise, Lemma~\ref{lem:neff_mass} gives
\(
\Pr_{C\sim p(\cdot\mid r)}[C\notin S]\ge 1-a=(1-a)_+.
\)
Thus
\[
\mathbb{E}[F_r(\theta^+)-F_r(\theta)]
\;\ge\;
(1-a)_+\Big(\rho\kappa-(1-\rho)\big(\eta G+\tfrac{L\eta^2}{2}\big)\Big),
\]
which is the claimed inequality.

\paragraph{Monotonicity in \(N_{\mathrm{eff}}(r)\) and positivity.}
Since \(a=\sqrt{m/N_{\mathrm{eff}}(r)}\) decreases as \(N_{\mathrm{eff}}(r)\) increases, the factor \((1-a)_+\)
is nondecreasing in \(N_{\mathrm{eff}}(r)\).
Therefore, whenever
\[
B:=\rho\kappa-(1-\rho)\Big(\eta G+\frac{L\eta^2}{2}\Big)\ \ge\ 0,
\]
the lower bound \((1-a)_+B\) is nondecreasing in \(N_{\mathrm{eff}}(r)\).
Moreover, if \(B>0\), the lower bound becomes strictly positive as soon as \((1-a)_+>0\), i.e., whenever
\(N_{\mathrm{eff}}(r)>m\).
\end{proof}

\subsection{Experiment Details}
\label{app:exp_detail}
\subsubsection{Datasets}
\paragraph{C-STANCE.}
C-STANCE is a large-scale Chinese benchmark for zero-shot stance detection, where each example pairs a microblog post with a target and the model predicts a stance label (favor/against/neutral) toward that target, including targets not observed during training. In TRACE, C-STANCE is casted as a 3-way classification task and evaluate with accuracy.

\paragraph{FOMC.}
The FOMC dataset is constructed from Federal Open Market Committee communications and is annotated for monetary-policy stance, enabling a hawkish--dovish style classification task. TRACE uses this dataset as a domain-specific stance classification problem (English) and reports accuracy.

\paragraph{MeetingBank.}
MeetingBank is a meeting summarization dataset built from city council meetings, providing long-form transcripts together with professionally written minutes and aligned segment-level supervision via a divide-and-conquer alignment procedure. TRACE uses MeetingBank as an abstractive summarization task and evaluates generation quality with ROUGE-L.

\paragraph{Py150.}
Py150 is a corpus of 150,000 Python source files mined from GitHub under permissive licensing filters and quality controls. It is widely used as a standard benchmark for Python code completion. TRACE uses Py150 as a code generation/completion task and evaluates with the fuzzing accuracy.

\paragraph{ScienceQA.}
ScienceQA is a science question answering benchmark containing multiple-choice questions collected from school science curricula. TRACE uses ScienceQA as a discrete-answer QA task and reports accuracy.

\paragraph{NumGLUE-cm.}
NumGLUE is a suite of arithmetic-centric reasoning tasks. TRACE includes the NumGLUE \emph{Commonsense + Arithmetic} task (cm), which requires combining commonsense quantitative facts with simple arithmetic operations. TRACE evaluates this task using accuracy.

\paragraph{NumGLUE-ds.}
TRACE also includes the NumGLUE \emph{Domain Specific + Arithmetic} task (ds), which requires domain knowledge together with arithmetic reasoning. As with other discrete-answer tasks in TRACE, performance is reported using accuracy.

\paragraph{20Minuten.}
20Minuten is a German dataset collected from the Swiss news outlet \emph{20 Minuten}, pairing full news articles with simplified rewrites/summaries to support document-level text simplification. TRACE uses 20Minuten as a German generation task and evaluates outputs using SARI.

\subsubsection{Baselines}
\label{app:baselines}

\paragraph{SeqLoRA.}
SeqLoRA is an adapter-based continual learning baseline that equips the pretrained model with a \emph{single shared} set of LoRA adapters. The same LoRA parameters are trained sequentially across all tasks in the stream, while the pretrained backbone remains frozen.

\paragraph{LoRAMoE \cite{dou2024loramoe}.}
LoRAMoE replaces selected linear layers with a shared \emph{common} linear map plus \(K\) LoRA-style low-rank residual experts and a learned gate. Routing is computed from a single router input representation, and each token activates a sparse subset of experts; the output is the common branch plus the probability-weighted residual from the selected experts. During continual learning, we freeze the pretrained backbone and train only the residual experts and gate parameters. We match LoRAMoE to MH-MoE in activated parameter count per token.

\paragraph{EWC \cite{kirkpatrick2017overcoming}.}
Elastic Weight Consolidation (EWC) is a regularization-based baseline that estimates parameter importance after each task (via a diagonal Fisher approximation) and penalizes changes to high-importance parameters on subsequent tasks. We apply EWC on top of the SeqLoRA setup.

\paragraph{GEM \cite{lopez2017gradient}.}
Gradient Episodic Memory (GEM) maintains a small episodic memory of samples from previous tasks and adjusts each update so it does not increase the loss on the stored memory samples. We apply GEM on top of the SeqLoRA setup.

\paragraph{O-LoRA \cite{wang2023orthogonal}.}
O-LoRA is a LoRA-based continual learning baseline that adds an orthogonality constraint/regularizer to reduce interference between sequential updates. In our implementation, we freeze the pretrained backbone and train only the LoRA parameters, following the original setup.

\subsubsection{Metrics}
Let $f_i(\mathbf{w}_j)$ denote the prediction performance on task $i$ (e.g., accuracy, SARI, ROUGE-L) when evaluated using the model parameters after learning through task $j$, denoted by $\mathbf{w}_j$.

\paragraph{Overall Performance (OP).}
After training up to task $n$, we define the overall performance as the average score over all tasks seen so far:
\begin{equation}
\mathrm{OP}_n \;\triangleq\; \frac{1}{n}\sum_{i=1}^{n} f_i(\mathbf{w}_n).
\end{equation}
 $\mathrm{OP}_n$ measures how good the final model $\mathbf{w}_n$ is on the full set of tasks $\{1,\dots,n\}$. Higher $\mathrm{OP}_n$ means better overall learning quality.

\paragraph{Backward Transfer (BWT).}
We quantify forgetting via backward transfer, defined as the average performance drop on earlier tasks after learning all tasks:
\begin{equation}
\mathrm{BWT}_n \;\triangleq\; \frac{1}{n}\sum_{i=1}^{n}\Bigl(f_i(\mathbf{w}_i) - f_i(\mathbf{w}_n)\Bigr).
\end{equation}
For each task $i$, the term $f_i(\mathbf{w}_i)$ is the model's performance right after learning task $i$, while $f_i(\mathbf{w}_n)$ is its performance on task $i$ after subsequently learning tasks $i\!+\!1,\dots,n$. Thus, $\mathrm{BWT}_n$ measures \emph{retention}: larger values indicate more forgetting (bigger degradation on old tasks), while values closer to $0$ indicate better preservation of earlier-task performance.

\subsubsection{Implementation Details}
For Qwen3-0.6B, we train each task for 2 epochs with a learning rate of $1\times 10^{-4}$ under a cosine learning-rate schedule. 
For Qwen3-8B, we use the same learning rate and schedule, but train each task for 5 epochs. 
In all experiments, we fix the sequence length to 2048 and use a batch size of 10. 
We optimize with AdamW ($\texttt{weight\_decay}=0.01$, $\beta_1=0.9$, $\beta_2=0.95$, $\epsilon=10^{-6}$). 
We attach experts to the linear modules in the MLP layers and tune the LoRA rank so that the number of activated parameters is comparable across methods. 
For LoRAMoE, we use 4 experts per layer with top-1 routing. 
For MH-MoE, each head has 4 private experts and performs top-1 routing within its head-specific expert set.

Unless stated otherwise, we use the standard TRACE task order: C-STANCE $\rightarrow$ FOMC $\rightarrow$ MeetingBank $\rightarrow$ Py150 $\rightarrow$ ScienceQA $\rightarrow$ NumGLUE-cm $\rightarrow$ NumGLUE-ds $\rightarrow$ 20Minuten.
For the task-order ablation, we evaluate three alternative sequences:
\textbf{Order 1}: FOMC $\rightarrow$ C-STANCE $\rightarrow$ ScienceQA $\rightarrow$ Py150 $\rightarrow$ MeetingBank $\rightarrow$ NumGLUE-cm $\rightarrow$ NumGLUE-ds $\rightarrow$ 20Minuten;
\textbf{Order 2}: FOMC $\rightarrow$ C-STANCE $\rightarrow$ ScienceQA $\rightarrow$ Py150 $\rightarrow$ NumGLUE-ds $\rightarrow$ NumGLUE-cm $\rightarrow$ MeetingBank $\rightarrow$ 20Minuten;
\textbf{Order 3}: FOMC $\rightarrow$ C-STANCE $\rightarrow$ ScienceQA $\rightarrow$ Py150 $\rightarrow$ MeetingBank $\rightarrow$ NumGLUE-ds $\rightarrow$ NumGLUE-cm $\rightarrow$ 20Minuten.

\subsection{Additional Experiment results}

\subsubsection{Head-structured Features. (Section \ref{sec:attn_output_multifeature})}

\begin{figure*}[ht]
  \centering

  \begin{subfigure}[t]{0.38\linewidth}
    \centering
    \includegraphics[width=\linewidth]{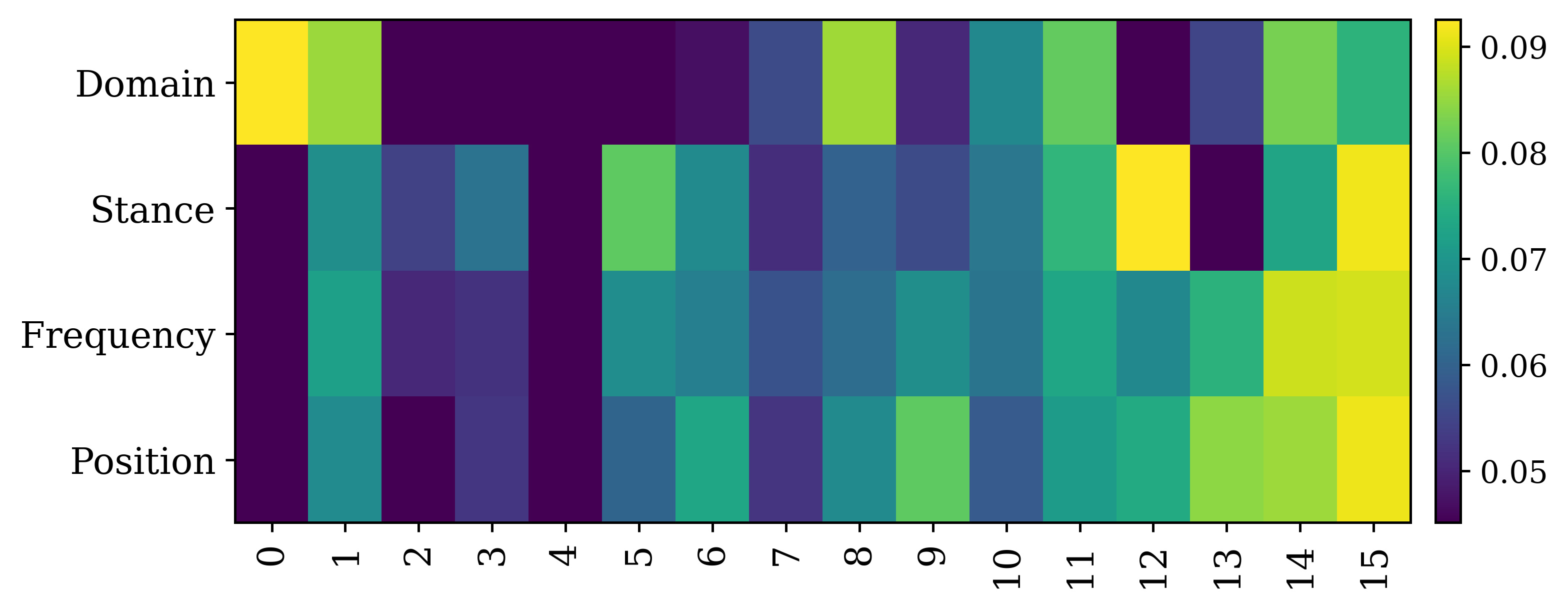}
    \caption{Layer 0}
  \end{subfigure}
  \begin{subfigure}[t]{0.38\linewidth}
    \centering
    \includegraphics[width=\linewidth]{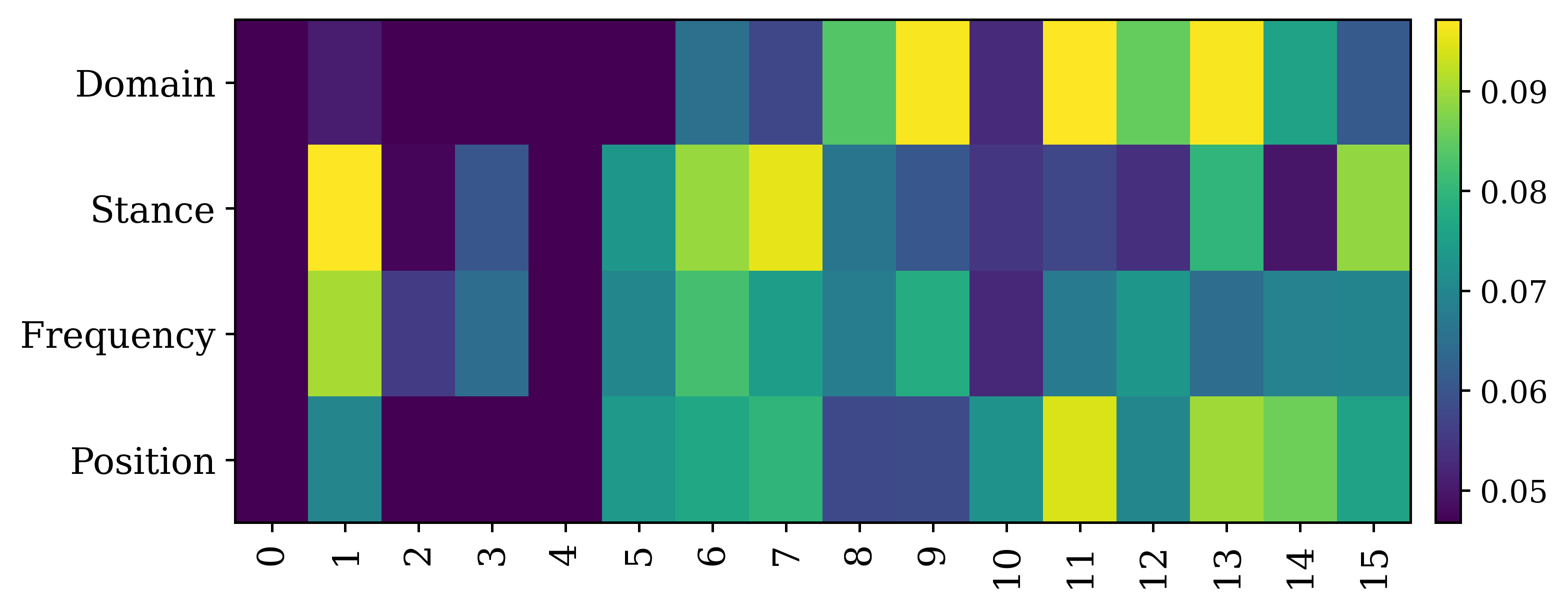}
    \caption{Layer 2}
  \end{subfigure}
  \begin{subfigure}[t]{0.38\linewidth}
    \centering
    \includegraphics[width=\linewidth]{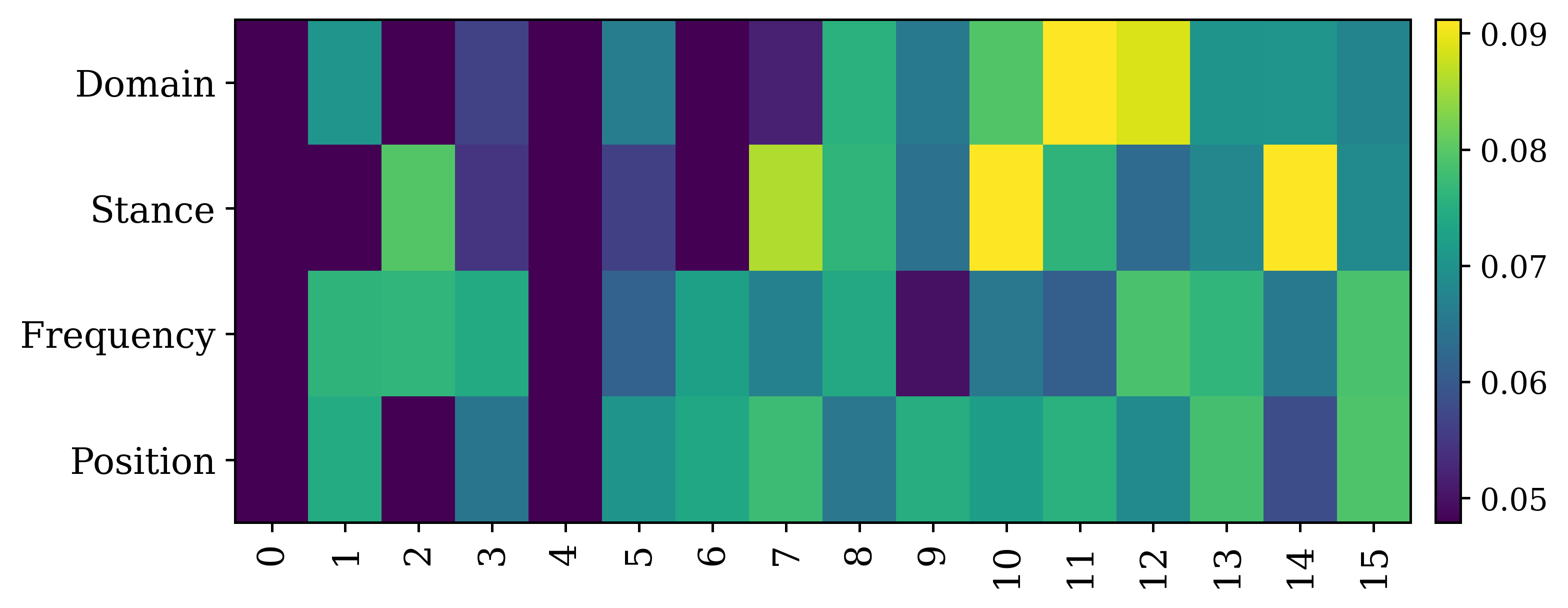}
    \caption{Layer 4}
  \end{subfigure}
  \begin{subfigure}[t]{0.38\linewidth}
    \centering
    \includegraphics[width=\linewidth]{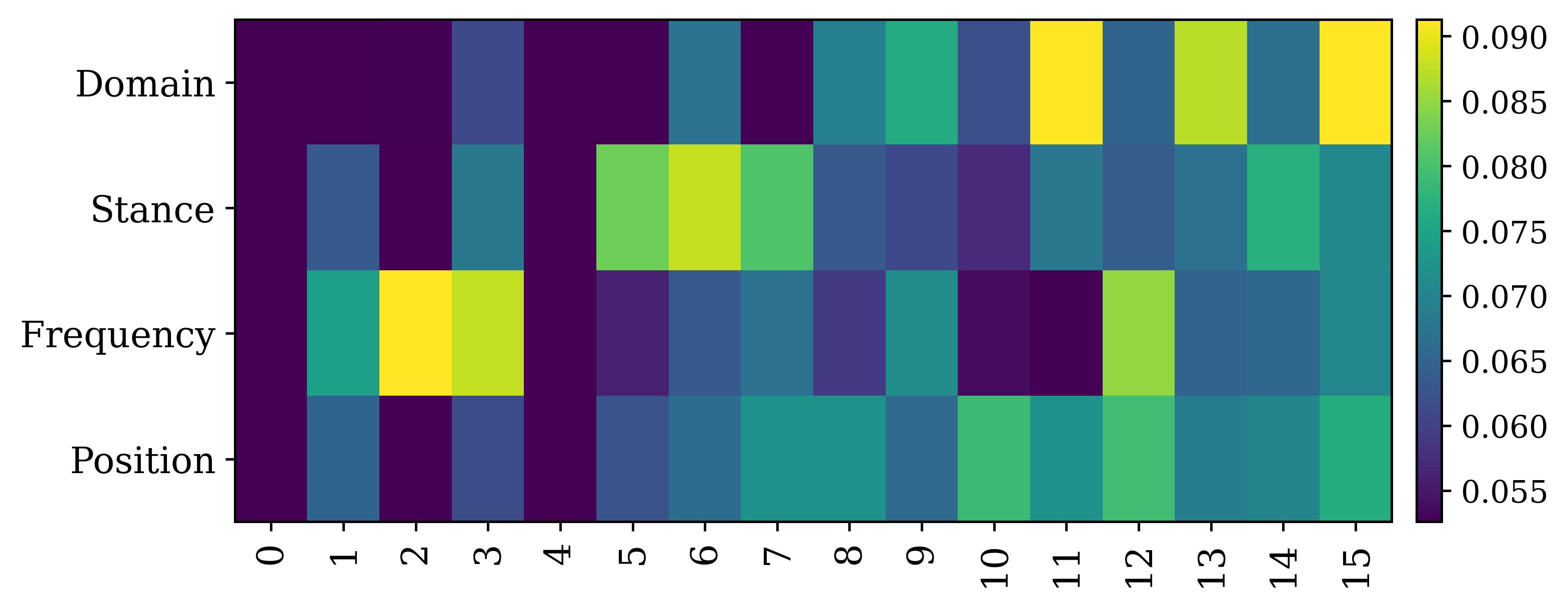}
    \caption{Layer 6}
  \end{subfigure}
  \begin{subfigure}[t]{0.38\linewidth}
    \centering
    \includegraphics[width=\linewidth]{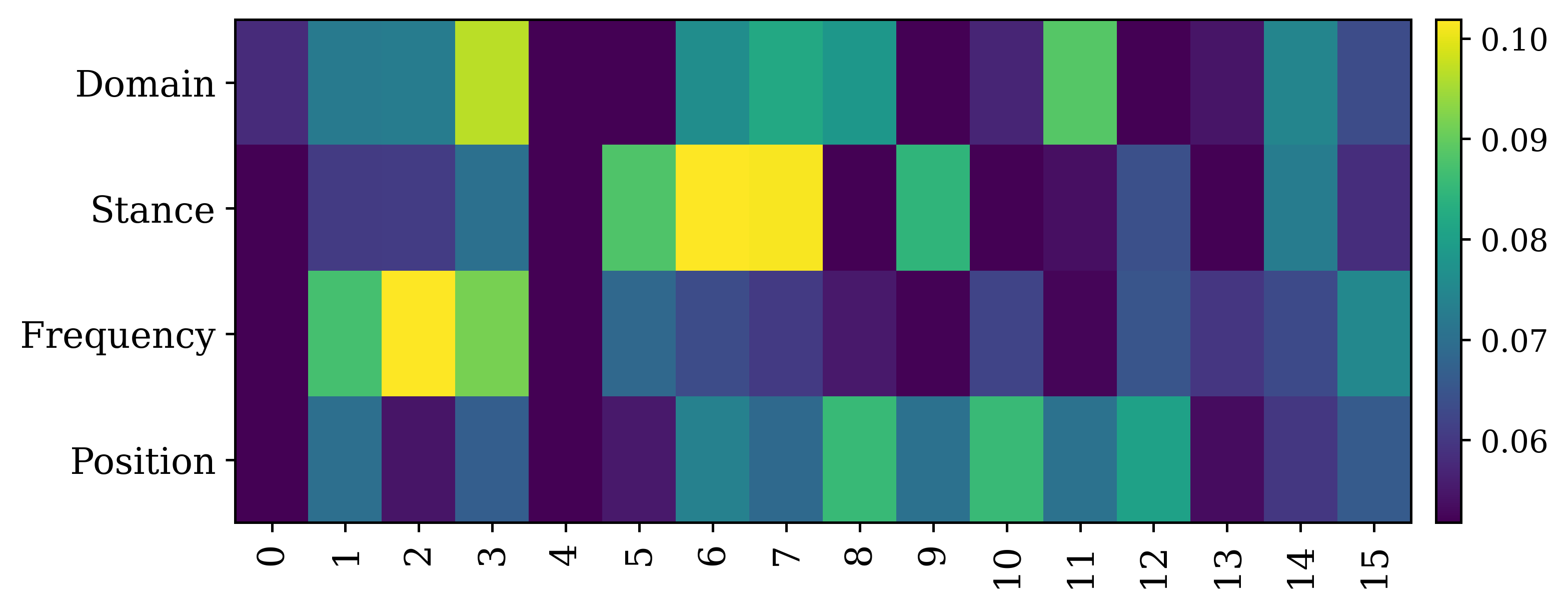}
    \caption{Layer 8}
  \end{subfigure}
  \begin{subfigure}[t]{0.38\linewidth}
    \centering
    \includegraphics[width=\linewidth]{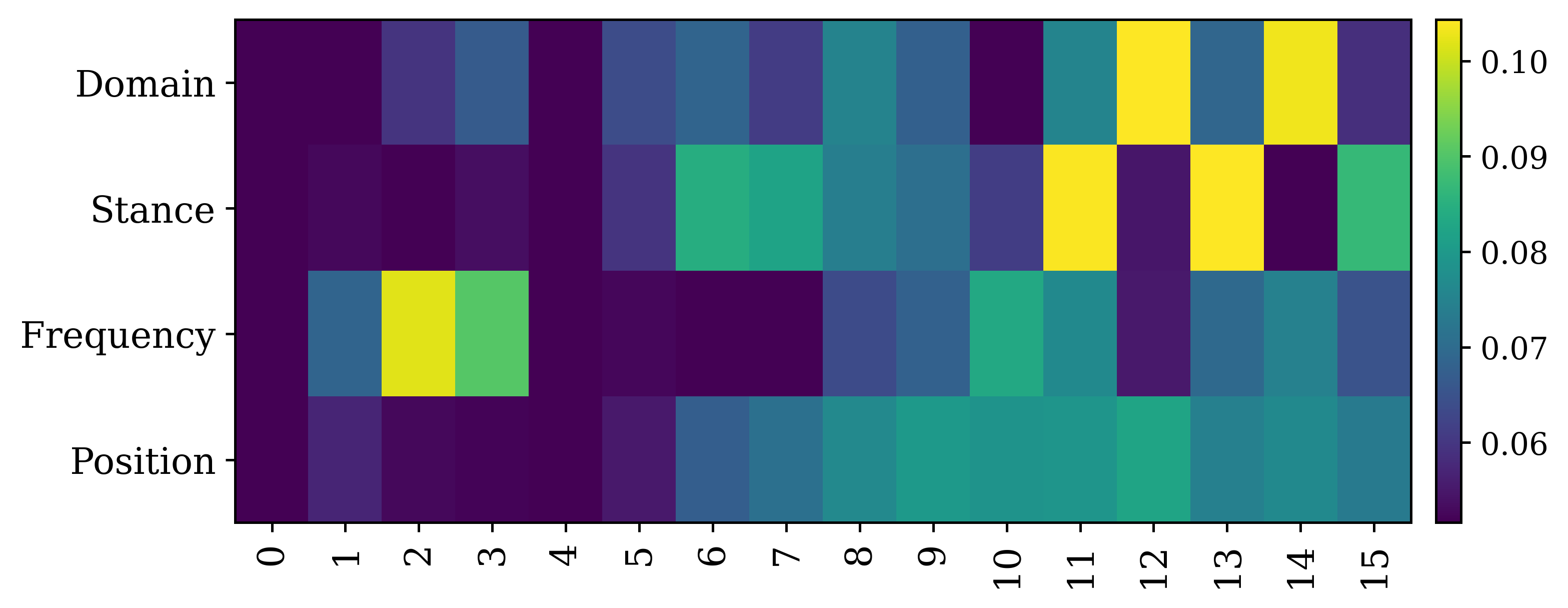}
    \caption{Layer 10}
  \end{subfigure}
  \begin{subfigure}[t]{0.38\linewidth}
    \centering
    \includegraphics[width=\linewidth]{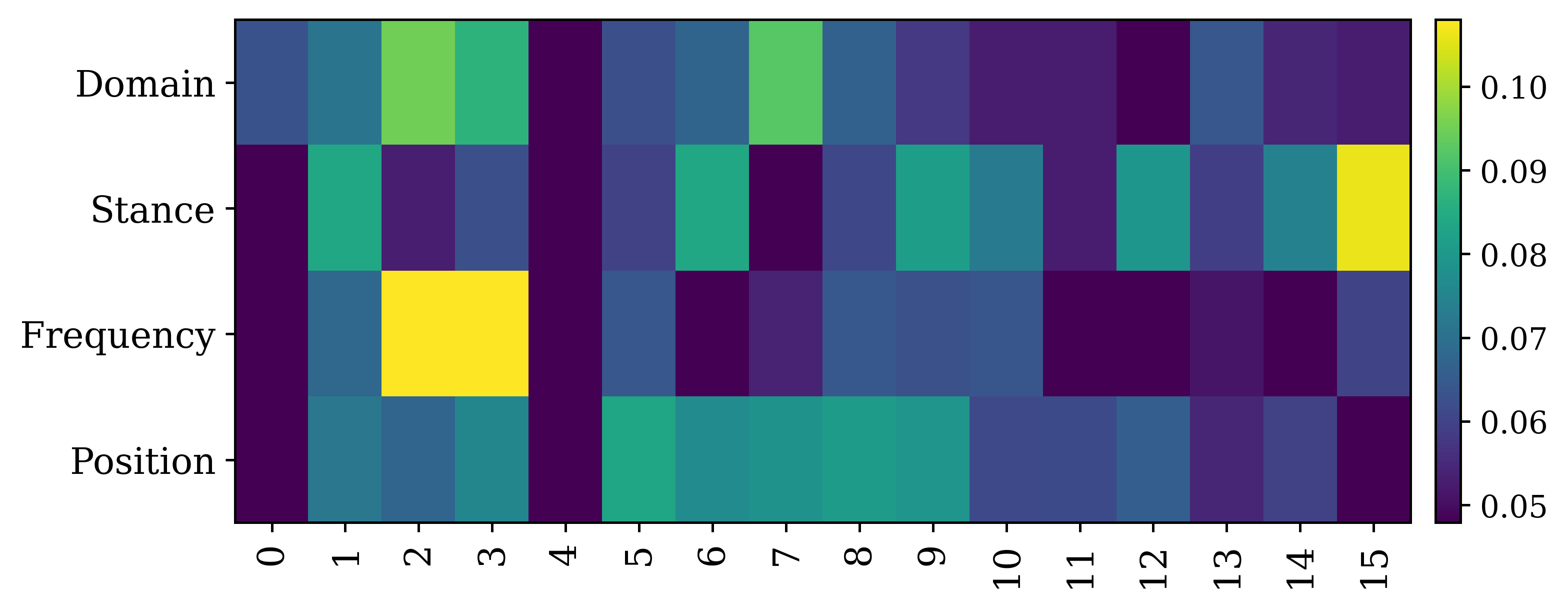}
    \caption{Layer 12}
  \end{subfigure}
  \begin{subfigure}[t]{0.38\linewidth}
    \centering
    \includegraphics[width=\linewidth]{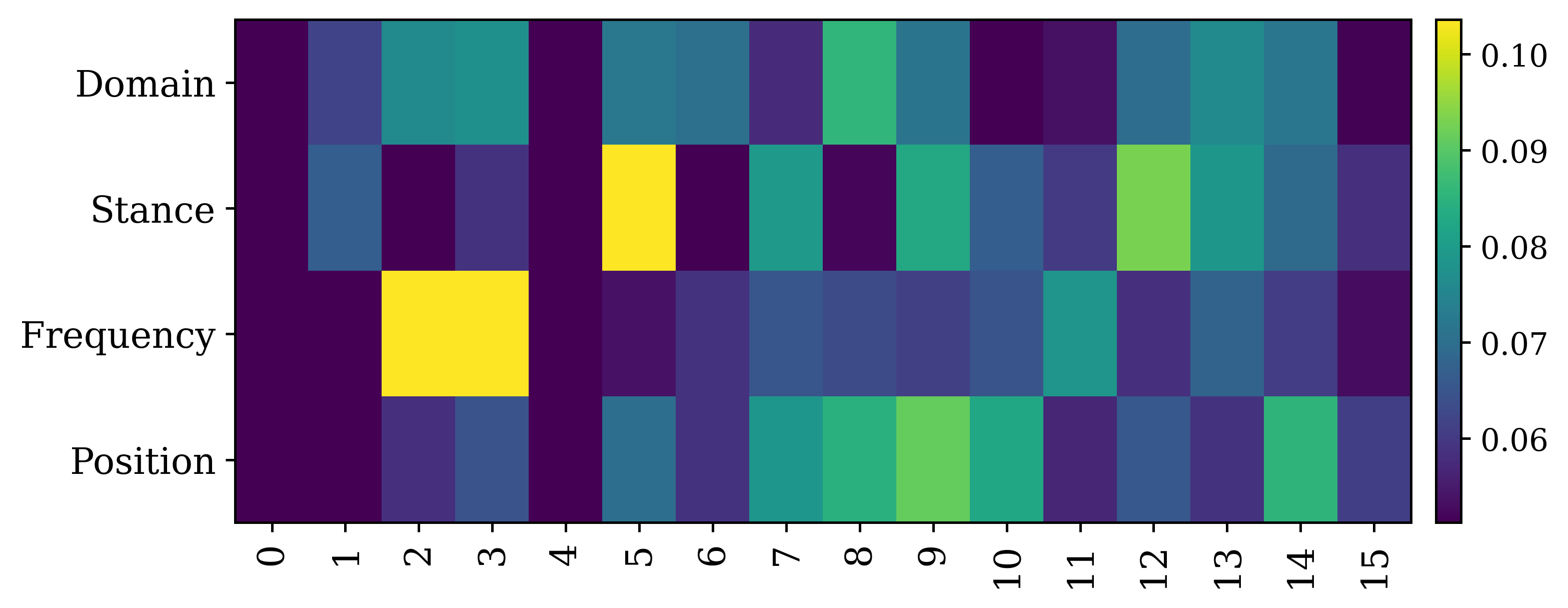}
    \caption{Layer 14}
  \end{subfigure}
  \begin{subfigure}[t]{0.38\linewidth}
    \centering
    \includegraphics[width=\linewidth]{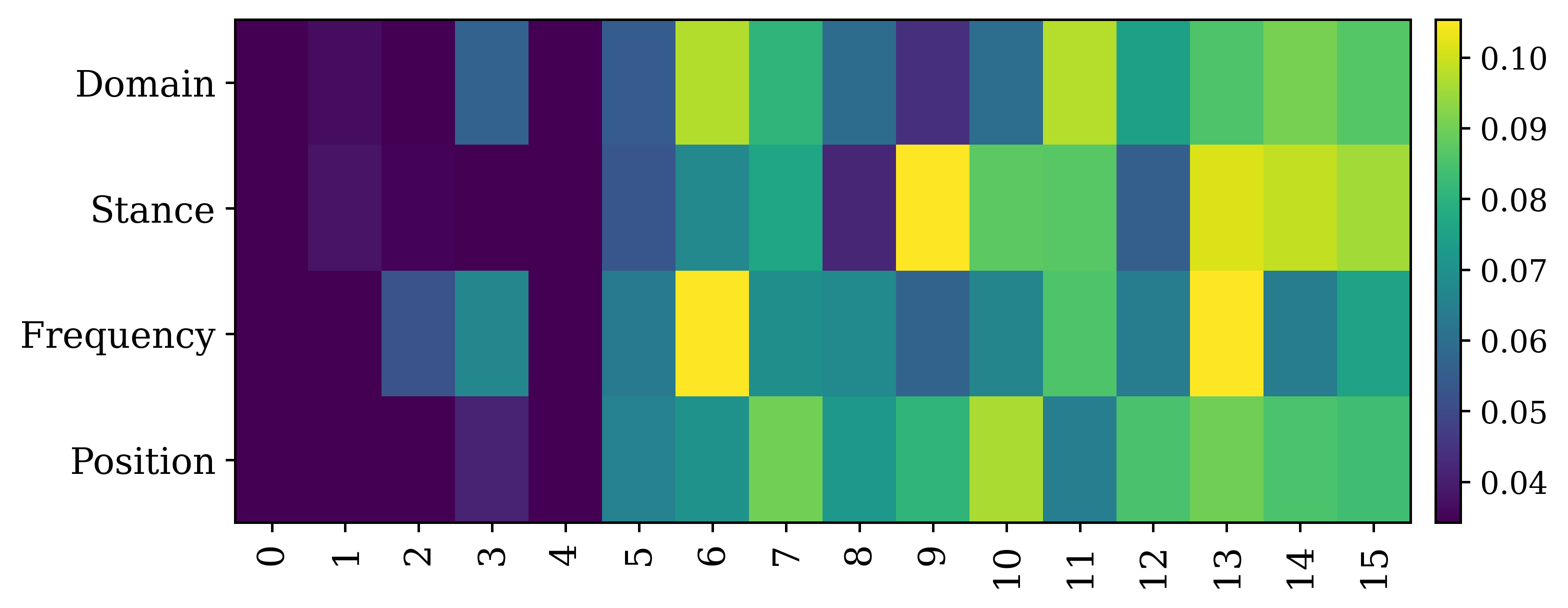}
    \caption{Layer 16}
  \end{subfigure}
  \begin{subfigure}[t]{0.38\linewidth}
    \centering
    \includegraphics[width=\linewidth]{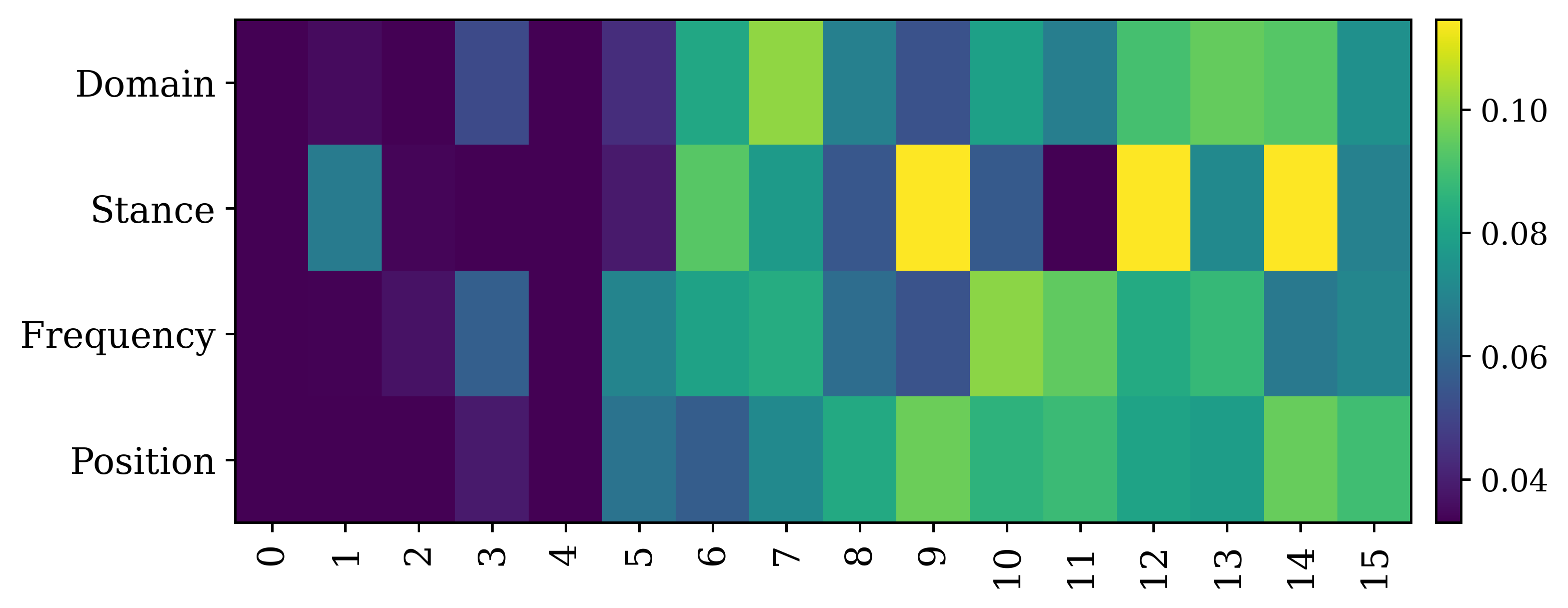}
    \caption{Layer 18}
  \end{subfigure}
  \begin{subfigure}[t]{0.38\linewidth}
    \centering
    \includegraphics[width=\linewidth]{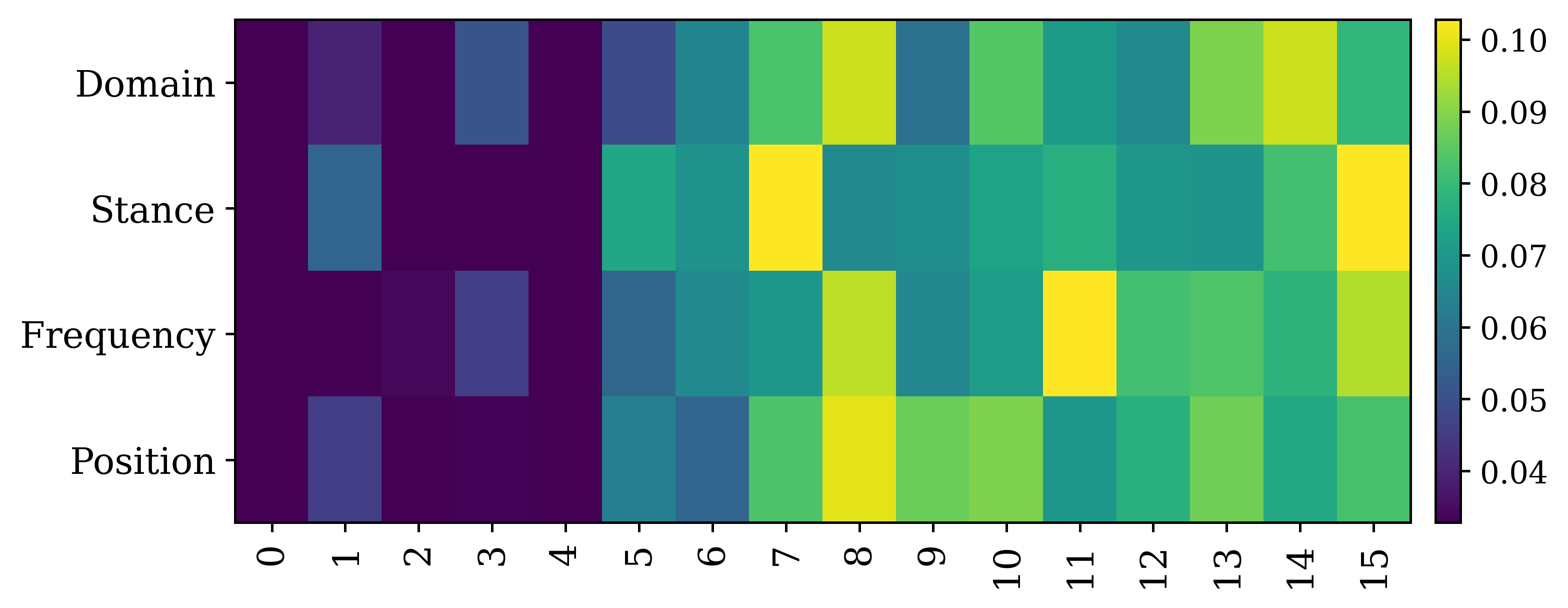}
    \caption{Layer 20}
  \end{subfigure}
  \begin{subfigure}[t]{0.38\linewidth}
    \centering
    \includegraphics[width=\linewidth]{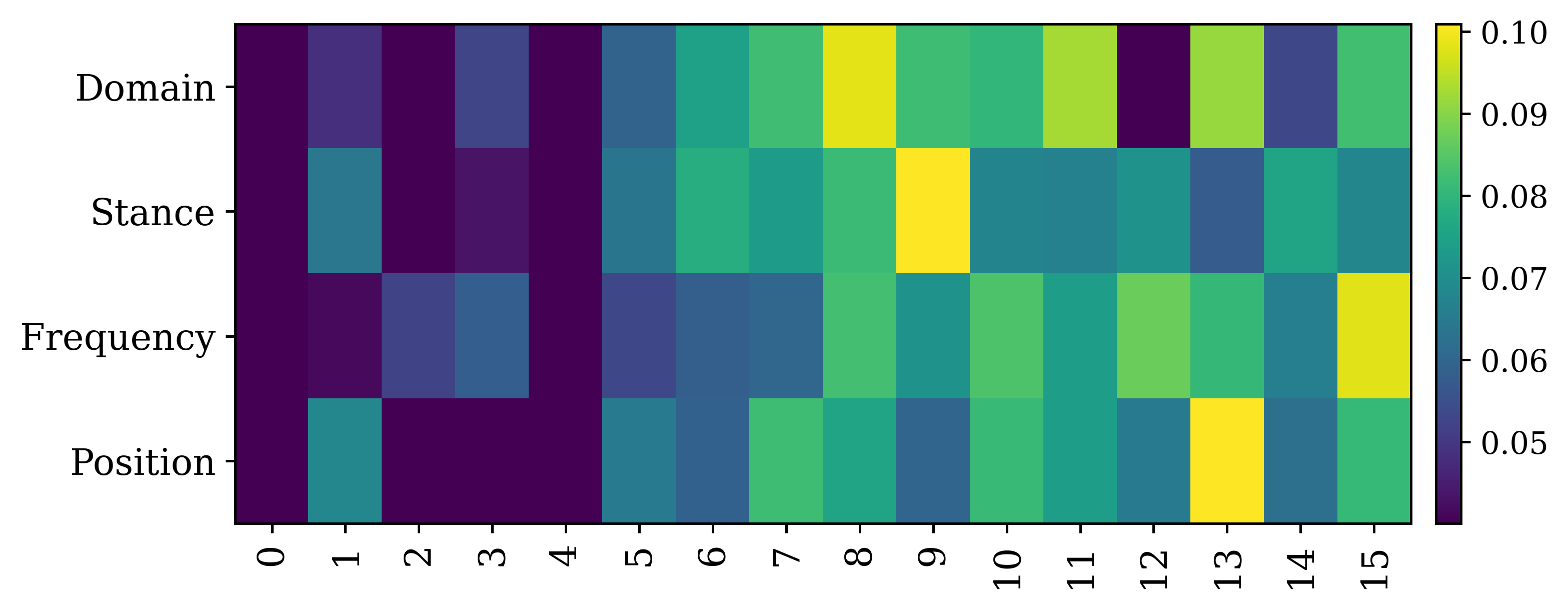}
    \caption{Layer 22}
  \end{subfigure}
  \begin{subfigure}[t]{0.38\linewidth}
    \centering
    \includegraphics[width=\linewidth]{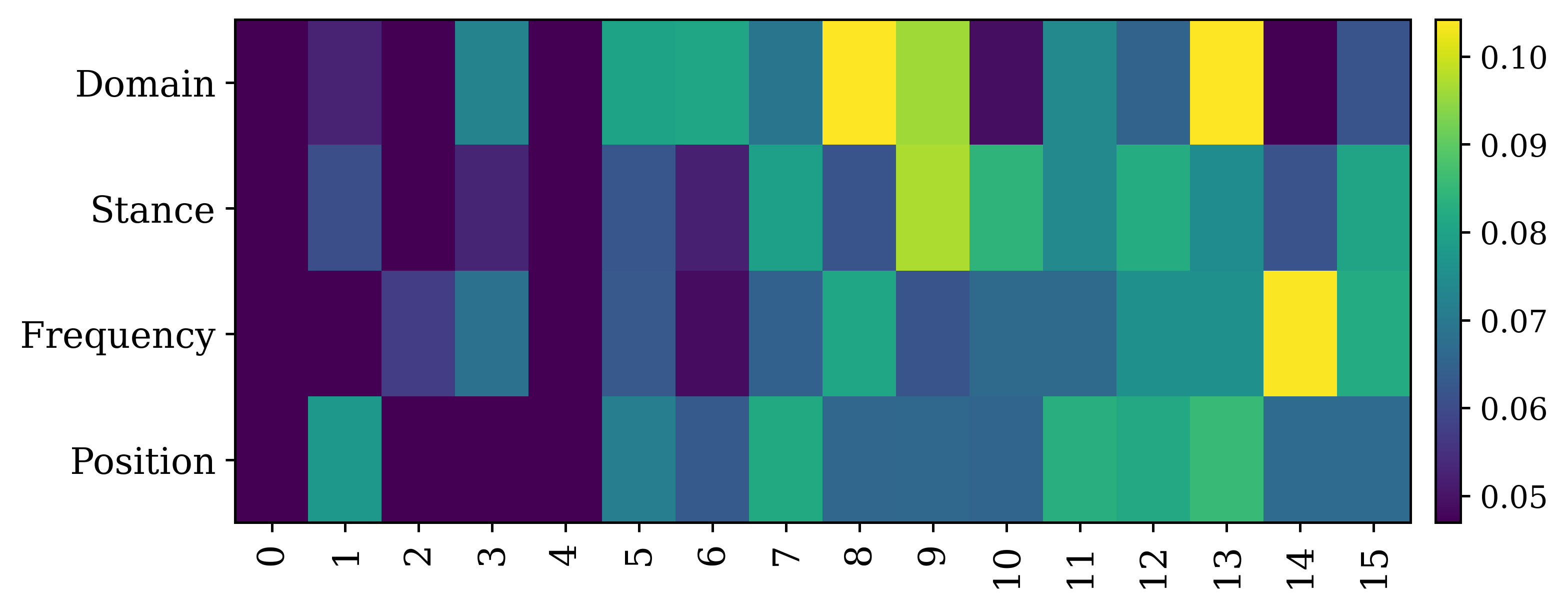}
    \caption{Layer 24}
  \end{subfigure}
  \begin{subfigure}[t]{0.38\linewidth}
    \centering
    \includegraphics[width=\linewidth]{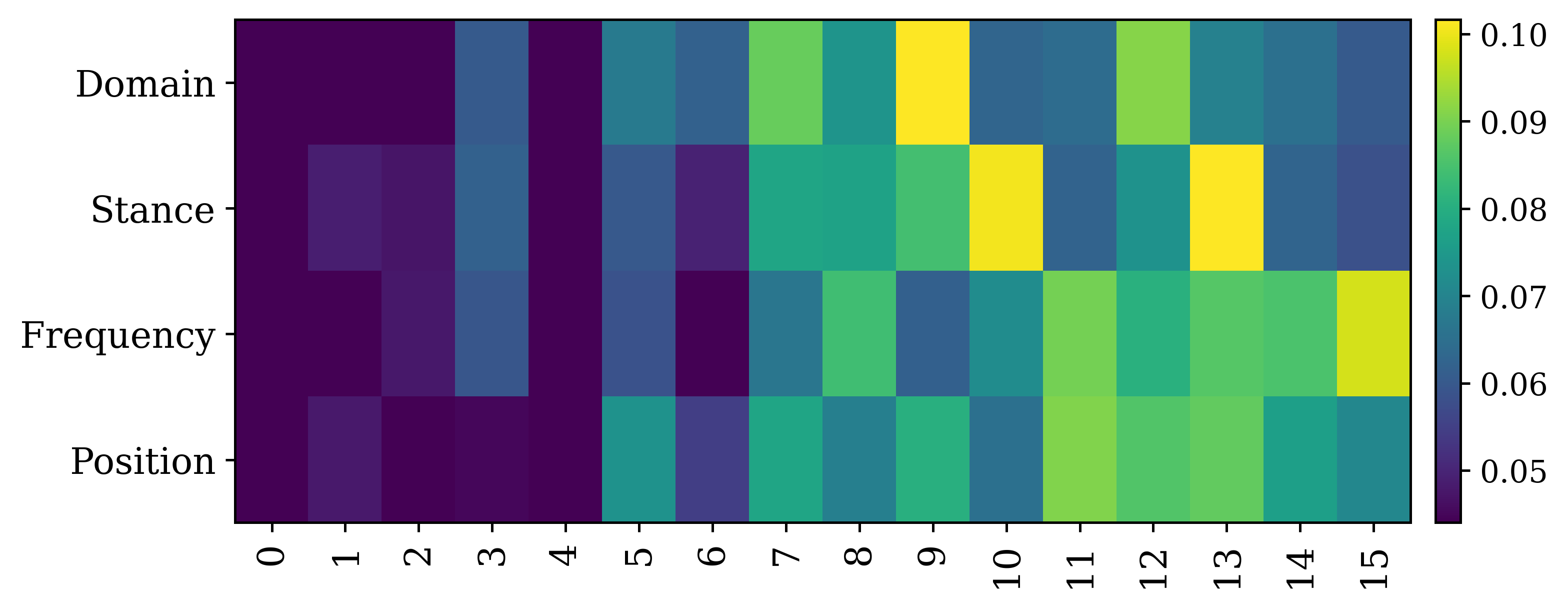}
    \caption{Layer 26}
  \end{subfigure}
  
  \caption{\textbf{Feature signals are head-structured across model layers.}}

  \label{fig:head_all}
\end{figure*}

\subsubsection{Within-composition coherence vs.\ cross-composition weak alignment. (Section \ref{sec:comp_grad})}

\begin{figure}[ht]
  \centering

  \begin{subfigure}[t]{0.335\linewidth}
    \centering
    \includegraphics[width=\linewidth]{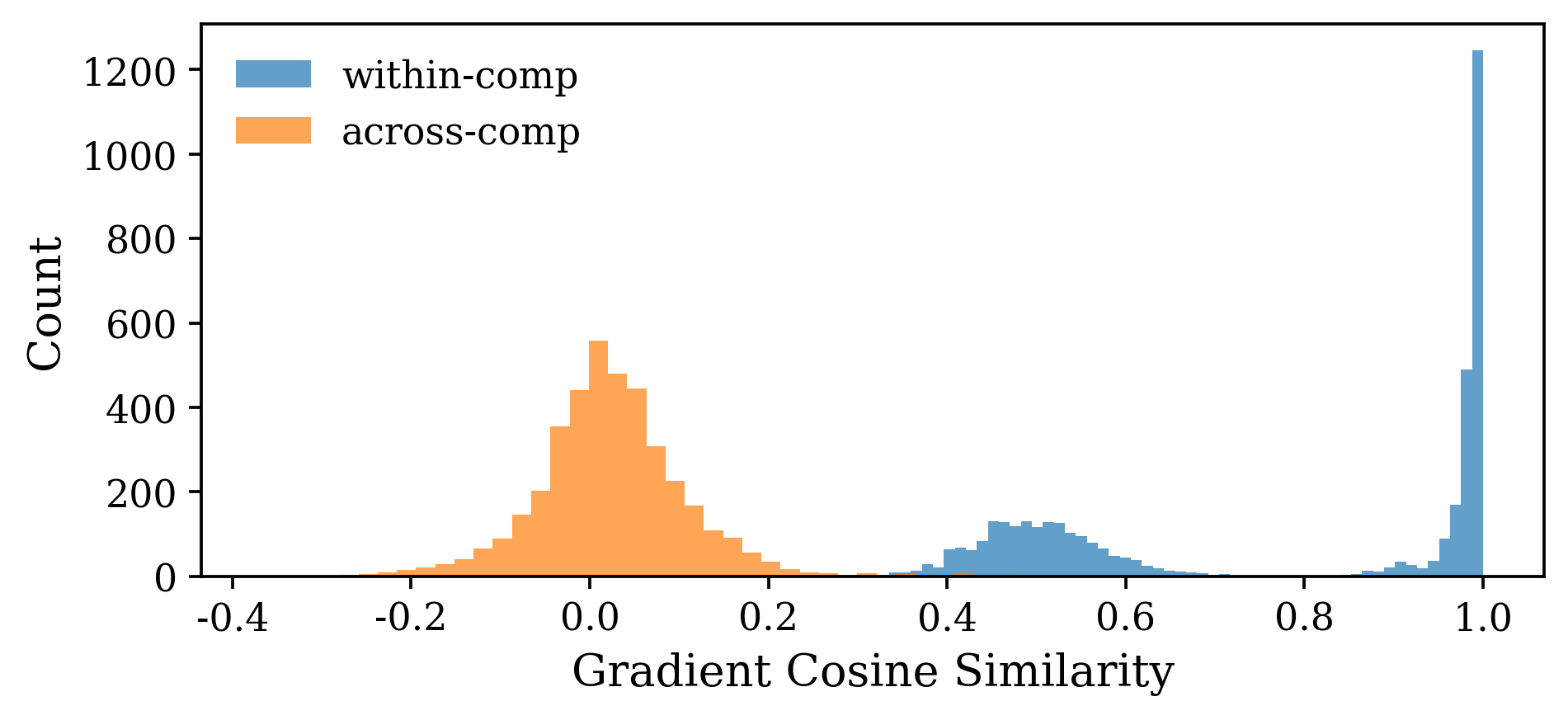}
    \caption{Layer 0}
  \end{subfigure}
  \begin{subfigure}[t]{0.335\linewidth}
    \centering
    \includegraphics[width=\linewidth]{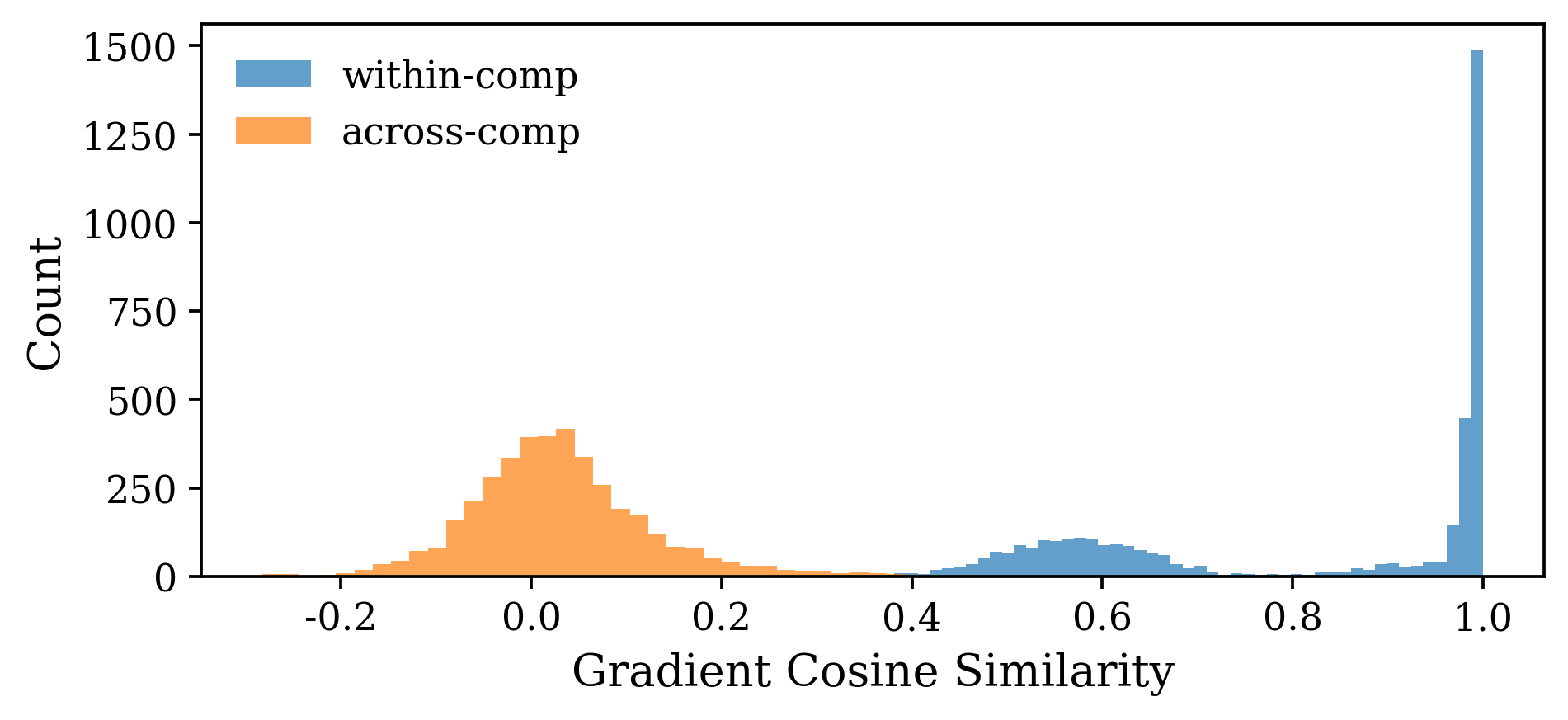}
    \caption{Layer 2}
  \end{subfigure}
  \begin{subfigure}[t]{0.335\linewidth}
    \centering
    \includegraphics[width=\linewidth]{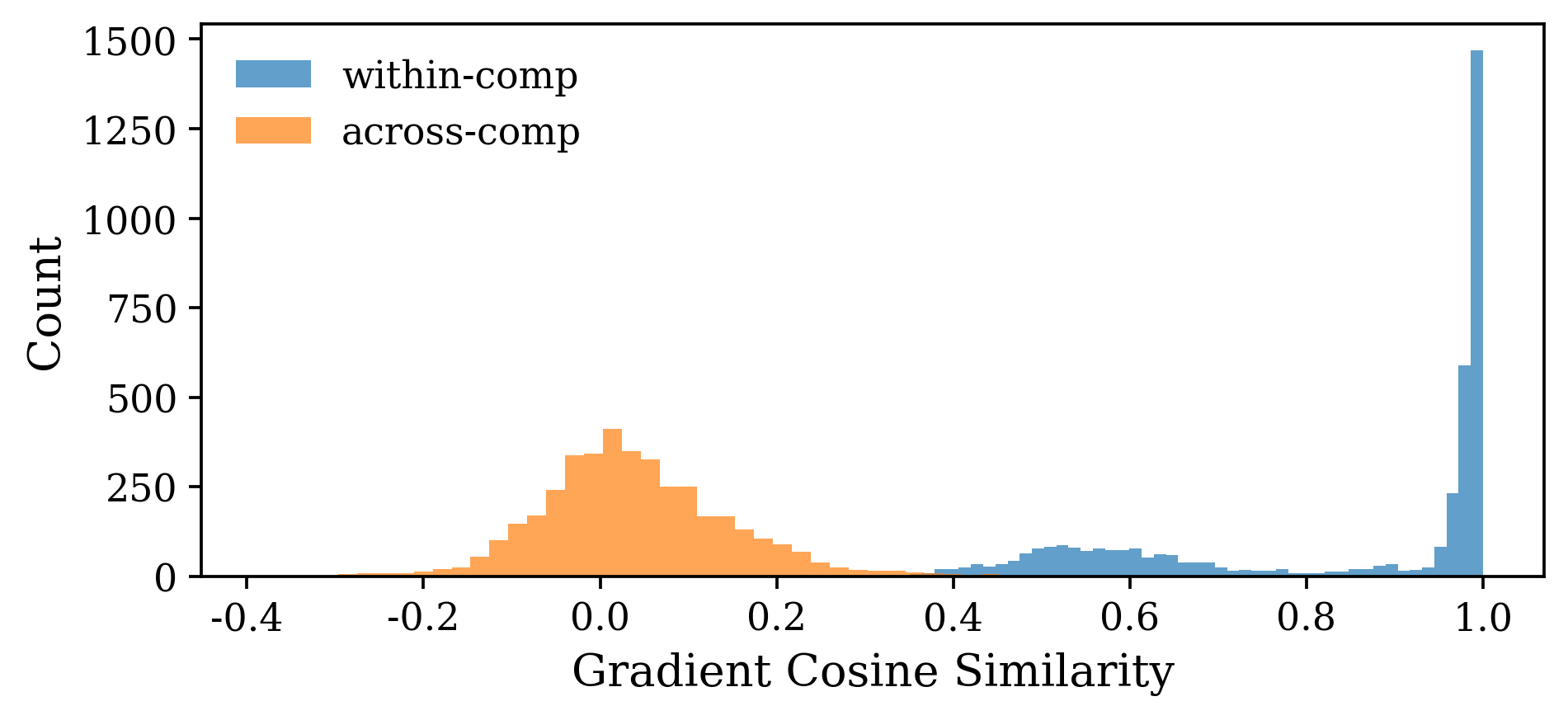}
    \caption{Layer 4}
  \end{subfigure}
  \begin{subfigure}[t]{0.335\linewidth}
    \centering
    \includegraphics[width=\linewidth]{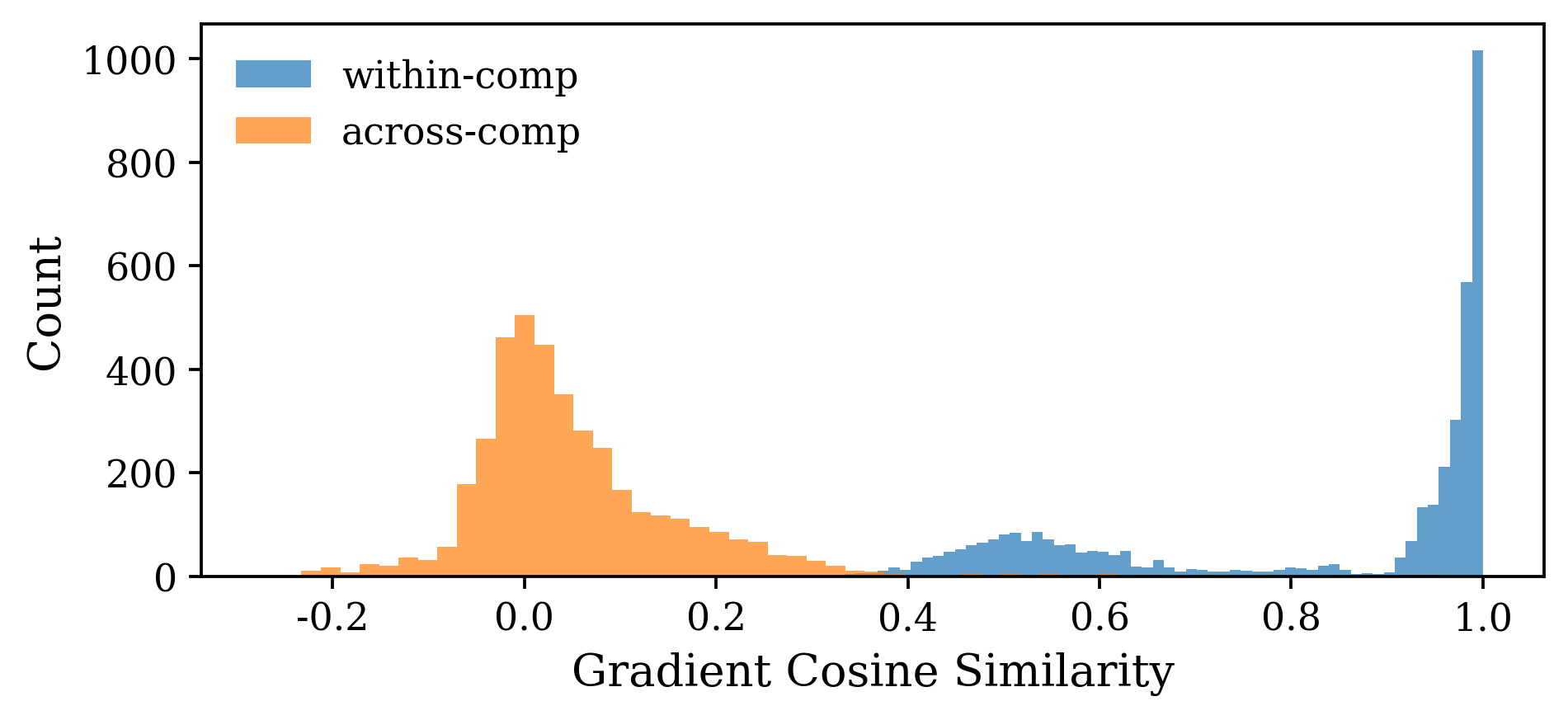}
    \caption{Layer 6}
  \end{subfigure}
  \begin{subfigure}[t]{0.335\linewidth}
    \centering
    \includegraphics[width=\linewidth]{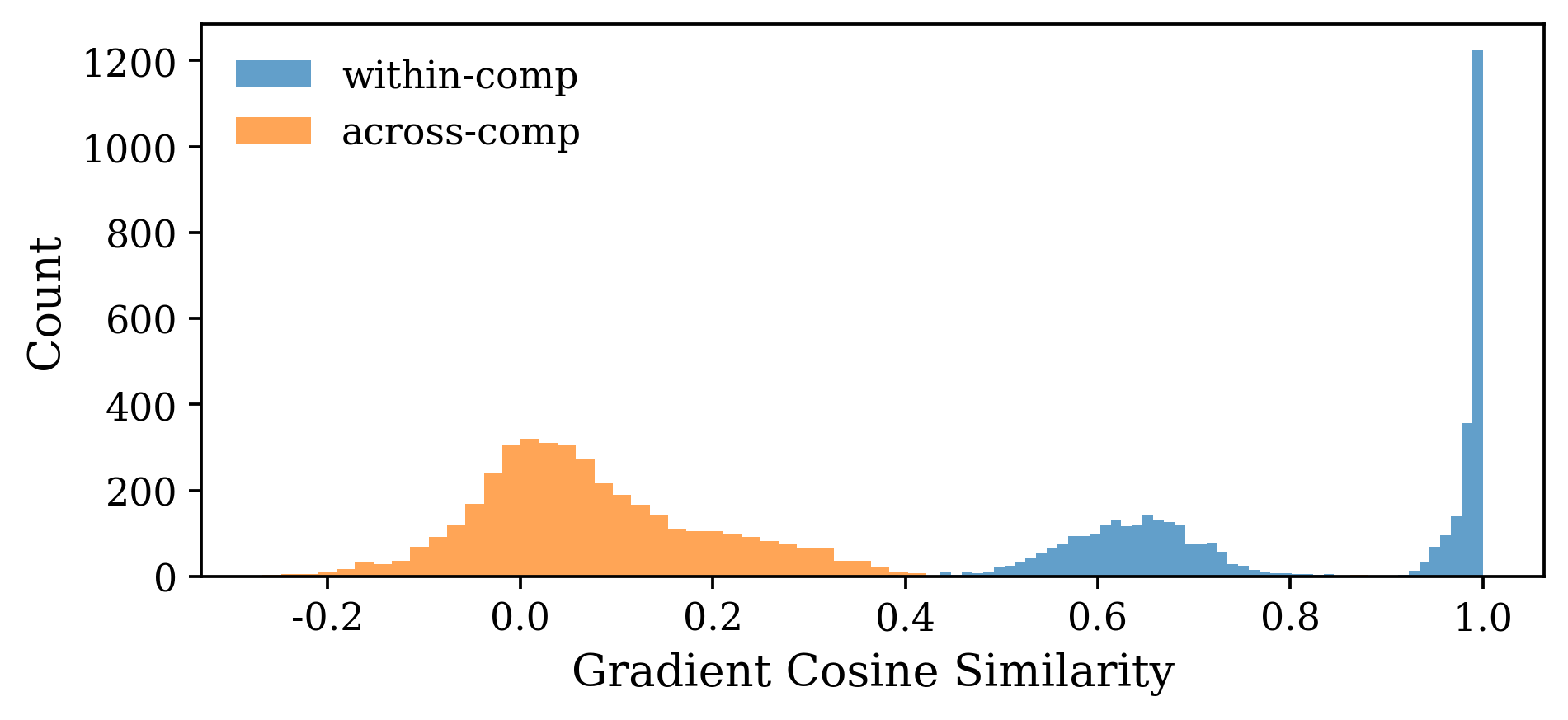}
    \caption{Layer 8}
  \end{subfigure}
  \begin{subfigure}[t]{0.335\linewidth}
    \centering
    \includegraphics[width=\linewidth]{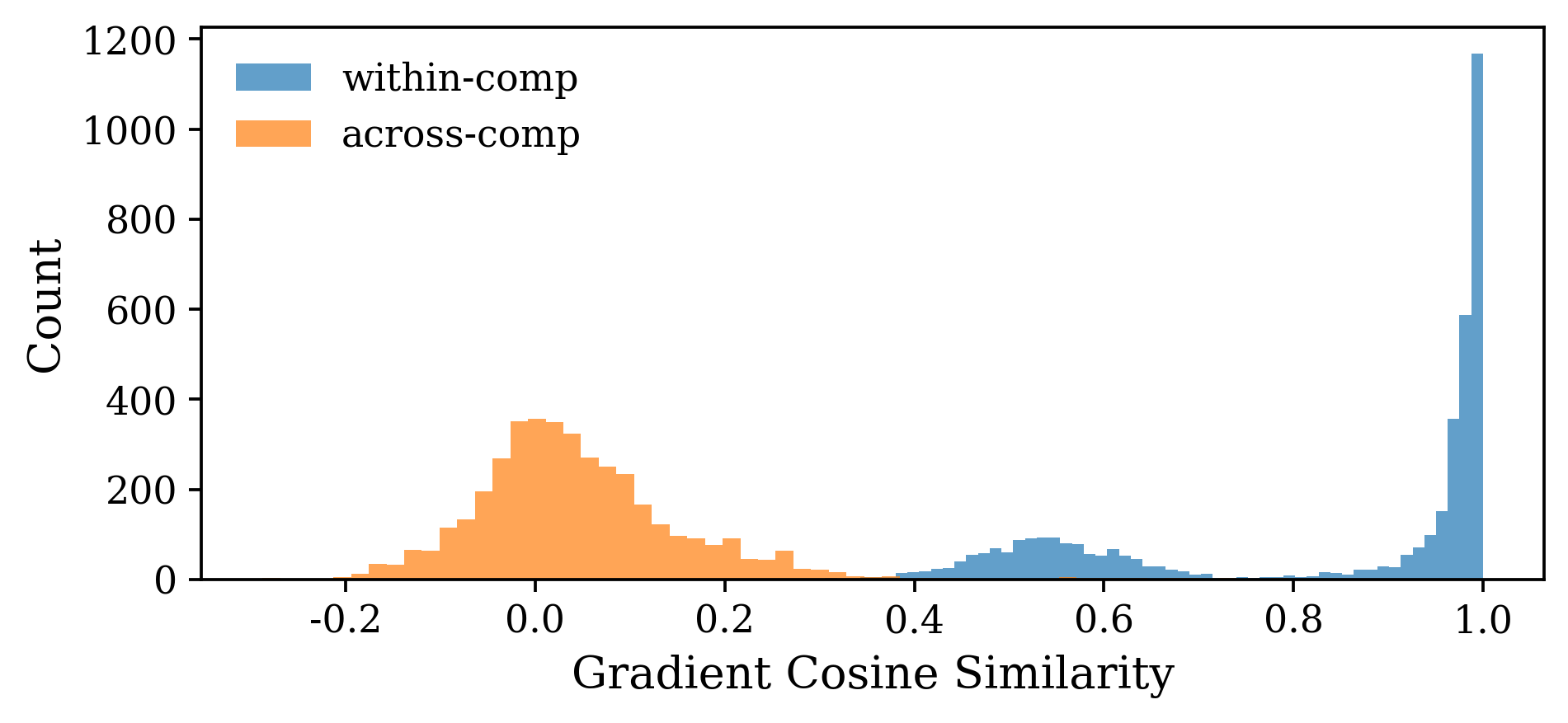}
    \caption{Layer 10}
  \end{subfigure}
  \begin{subfigure}[t]{0.335\linewidth}
    \centering
    \includegraphics[width=\linewidth]{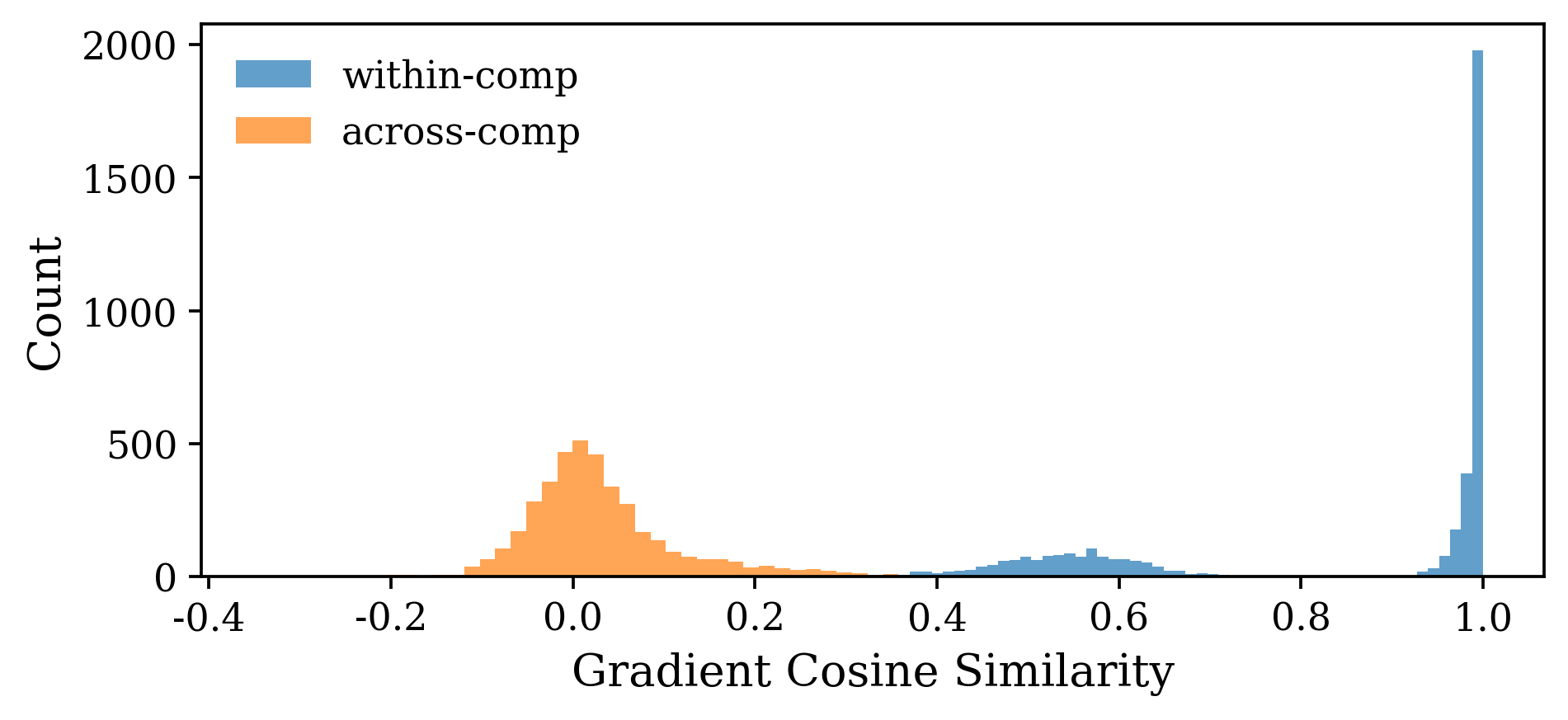}
    \caption{Layer 12}
  \end{subfigure}
  \begin{subfigure}[t]{0.335\linewidth}
    \centering
    \includegraphics[width=\linewidth]{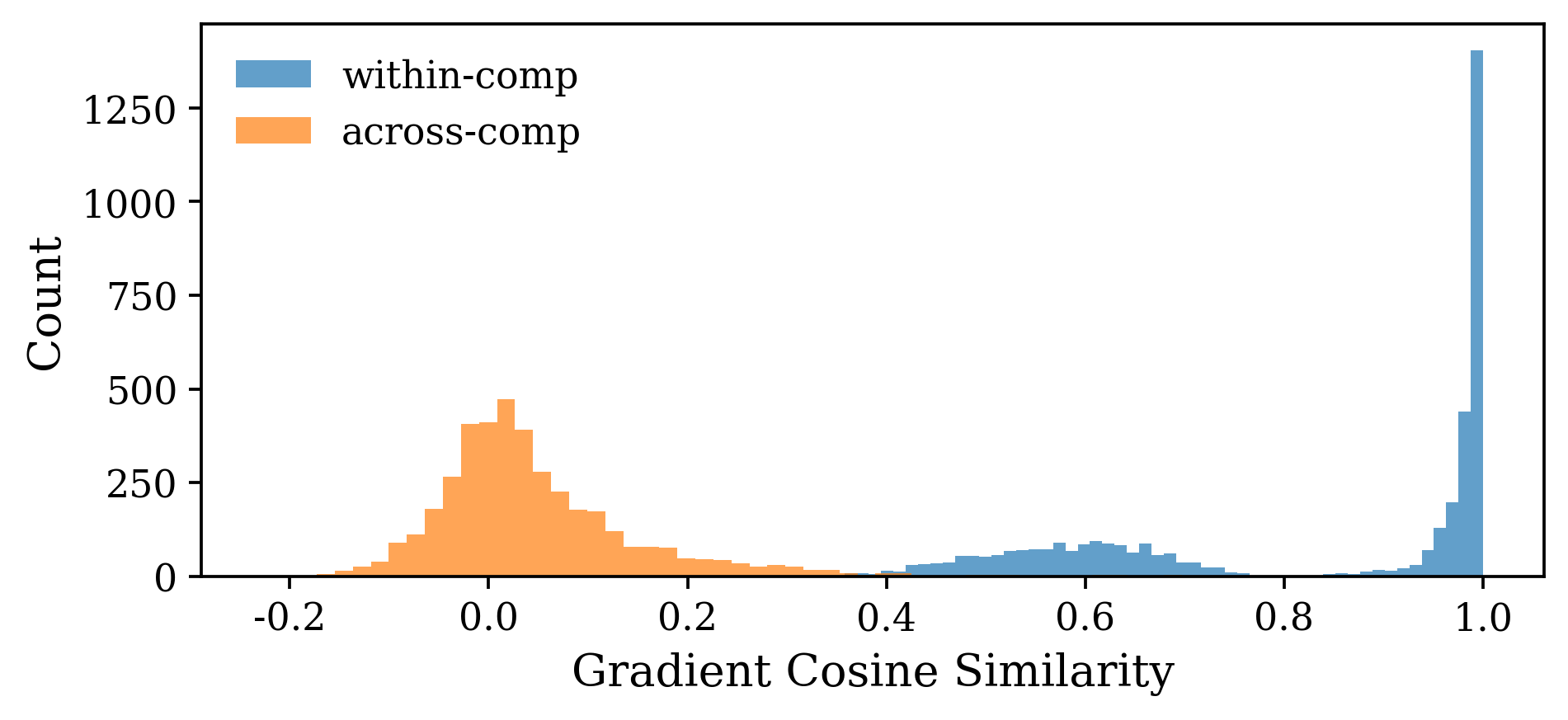}
    \caption{Layer 14}
  \end{subfigure}
  \begin{subfigure}[t]{0.335\linewidth}
    \centering
    \includegraphics[width=\linewidth]{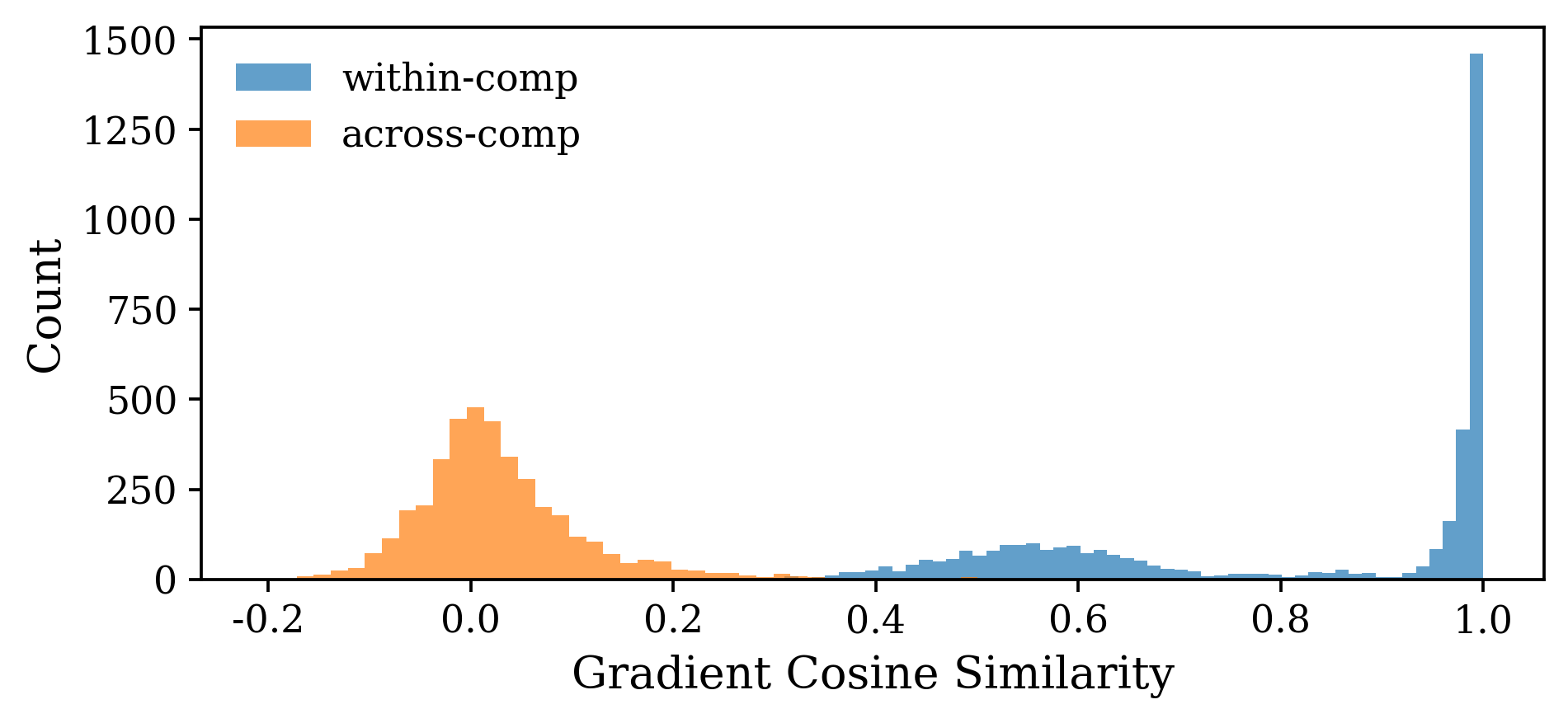}
    \caption{Layer 16}
  \end{subfigure}
  \begin{subfigure}[t]{0.335\linewidth}
    \centering
    \includegraphics[width=\linewidth]{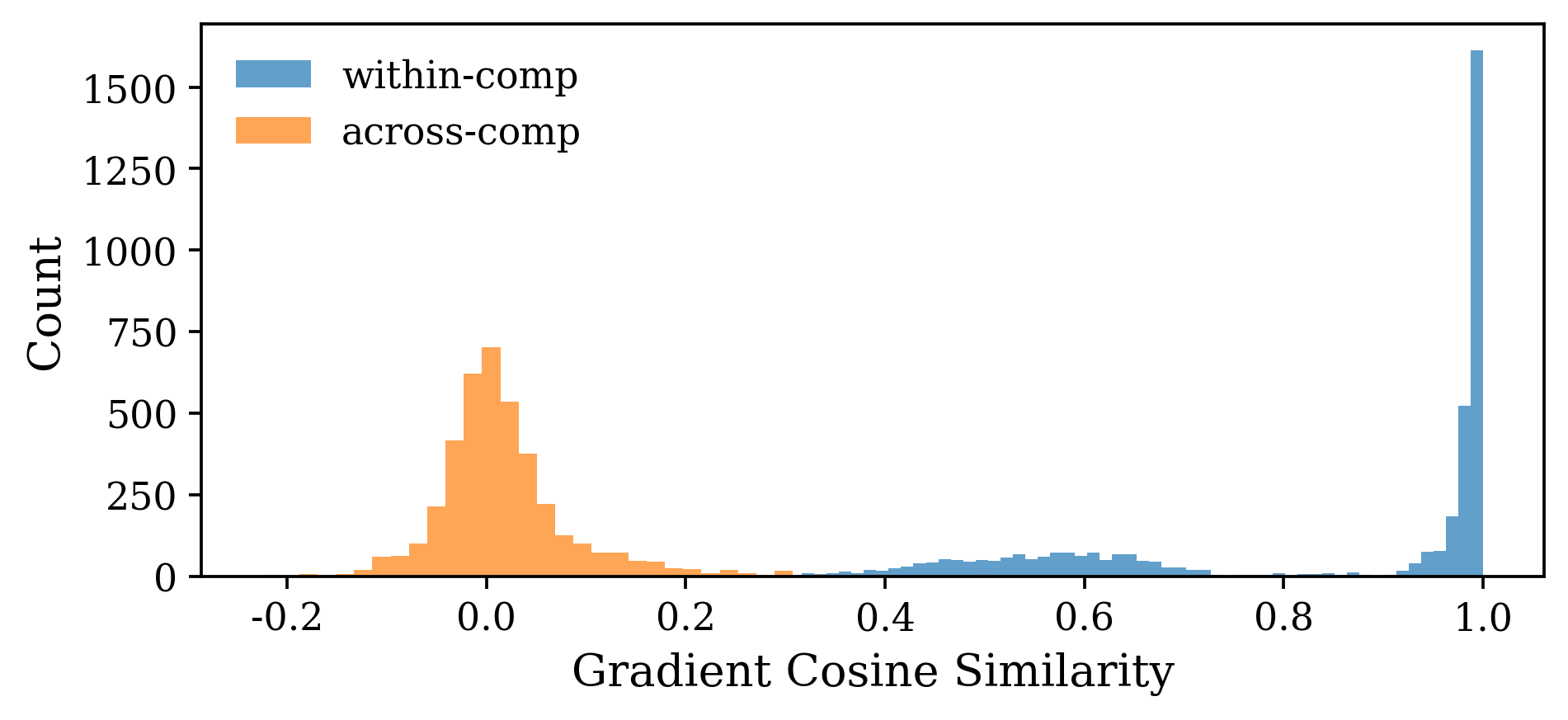}
    \caption{Layer 18}
  \end{subfigure}
  \begin{subfigure}[t]{0.335\linewidth}
    \centering
    \includegraphics[width=\linewidth]{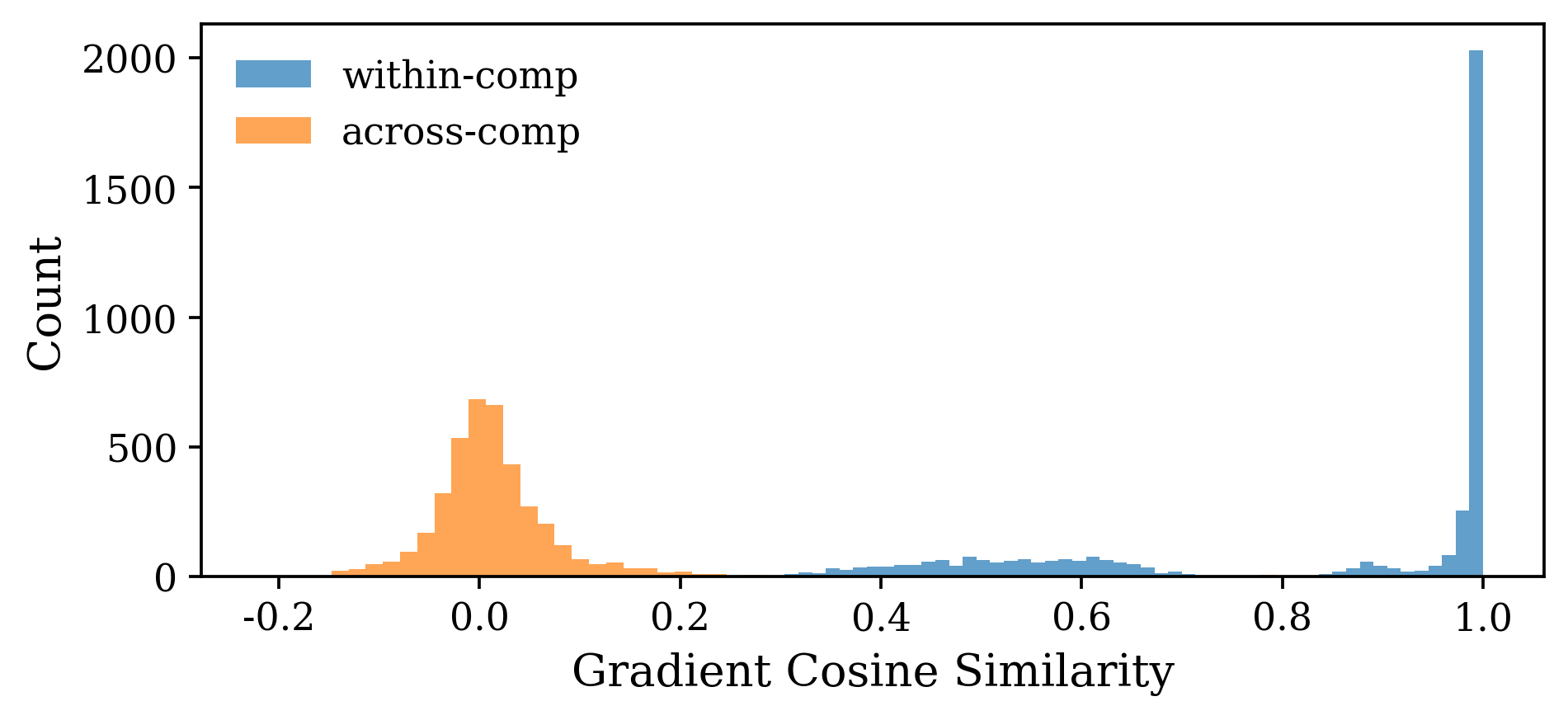}
    \caption{Layer 20}
  \end{subfigure}
  \begin{subfigure}[t]{0.335\linewidth}
    \centering
    \includegraphics[width=\linewidth]{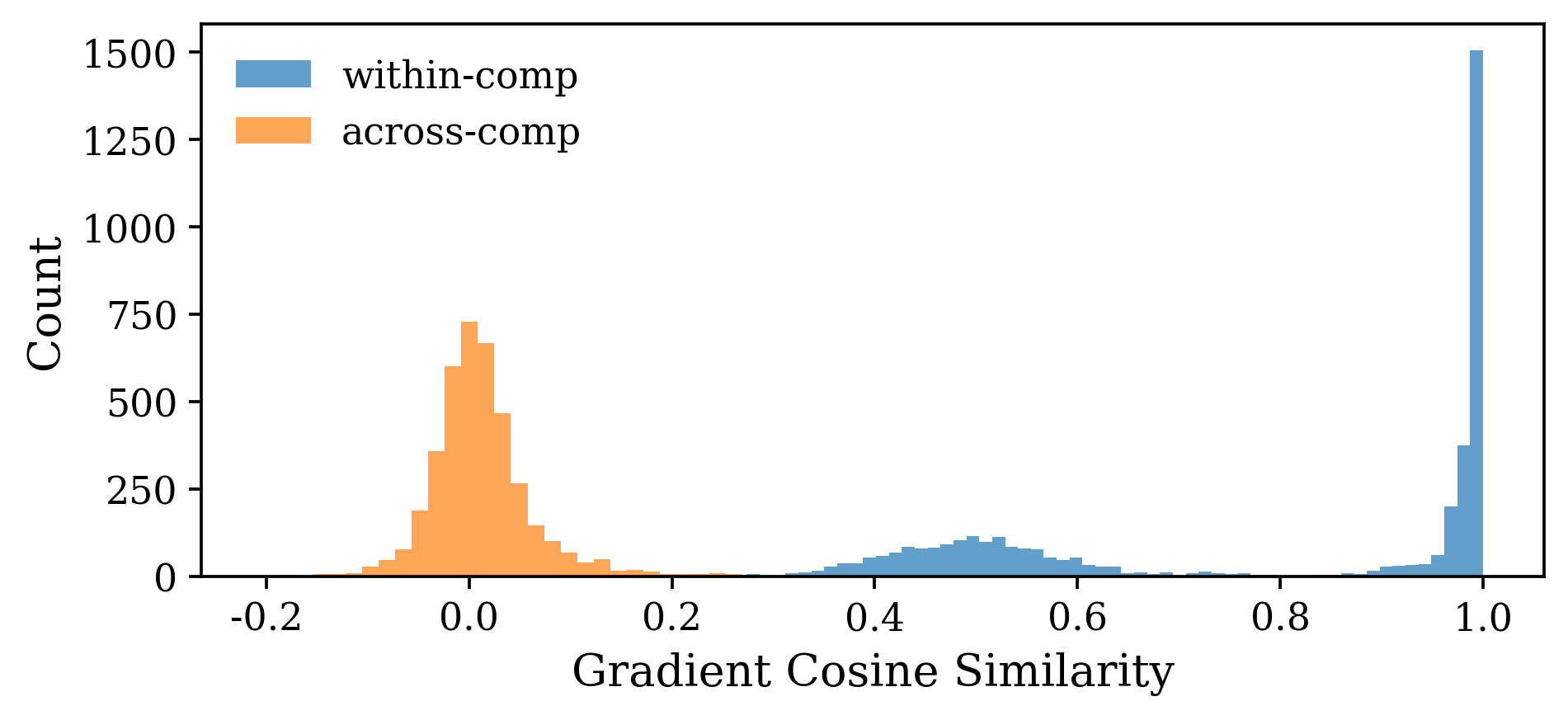}
    \caption{Layer 22}
  \end{subfigure}
  \begin{subfigure}[t]{0.335\linewidth}
    \centering
    \includegraphics[width=\linewidth]{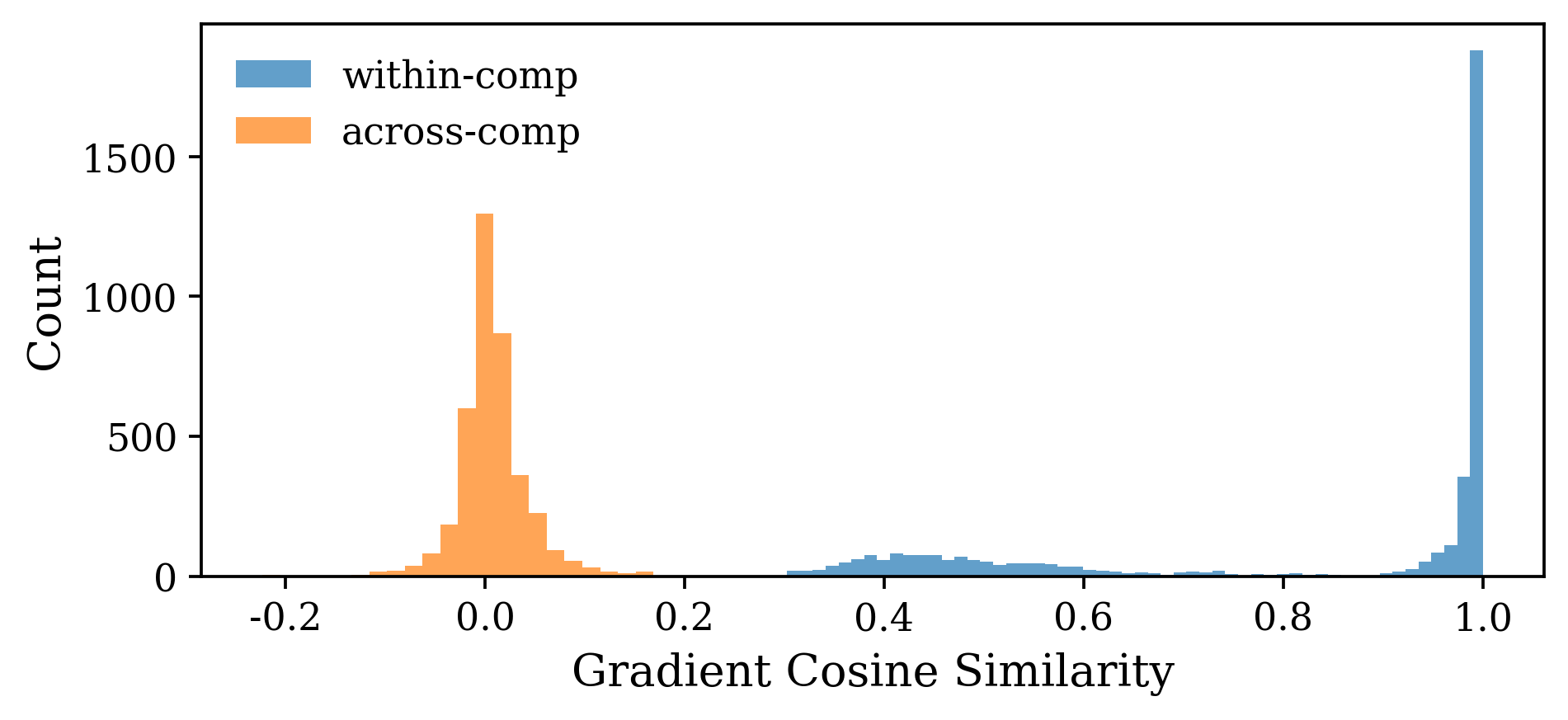}
    \caption{Layer 24}
  \end{subfigure}
  \begin{subfigure}[t]{0.335\linewidth}
    \centering
    \includegraphics[width=\linewidth]{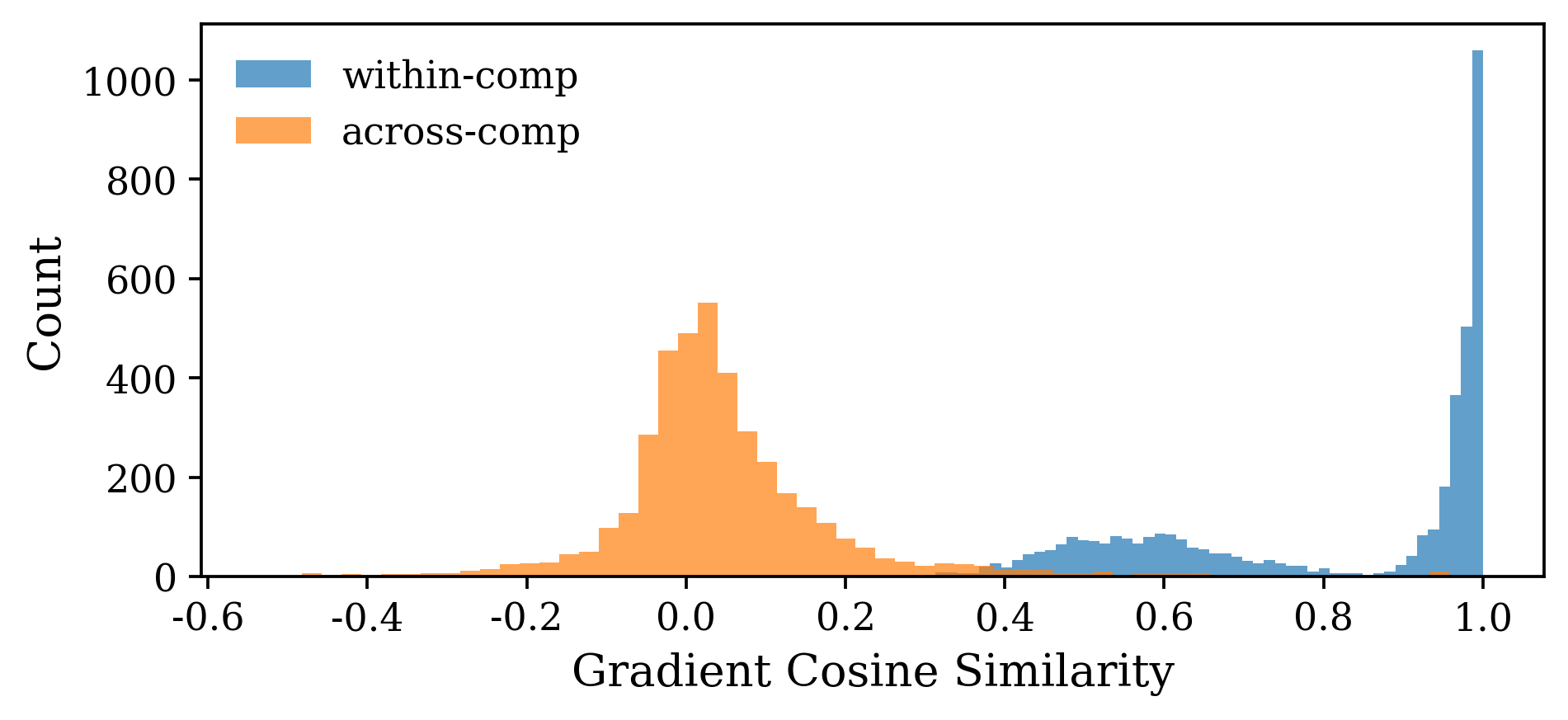}
    \caption{Layer 26}
  \end{subfigure}
  
  \caption{\textbf{Different feature compositions induce distinct gradient directions across layers.}}

  \label{fig:grad_all}
\end{figure}